\documentclass[12pt]{l4dc2023}

\title[Geometric Properties of Constraint Set in Safe Optimization]{The Impact of the Geometric Properties of the Constraint Set in Safe Optimization with Bandit Feedback}
\usepackage{times}

\usepackage{amsmath}
\usepackage{mathtools}
\DeclareMathOperator*{\argmin}{arg\,min}
\DeclareMathOperator*{\argmax}{arg\,max}
\usepackage{algpseudocode}
\usepackage{cleveref}

\usepackage{wrapfig}
\SetAlgoCaptionSeparator{}

\author{%
 \Name{Spencer Hutchinson} \Email{shutchinson@ucsb.edu}\\
 \addr University of California, Santa Barbara
 \AND
 \Name{Berkay Turan} \Email{bturan@ucsb.edu}\\
 \addr University of California, Santa Barbara%
 \AND
 \Name{Mahnoosh Alizadeh} \Email{alizadeh@ucsb.edu}\\
 \addr University of California, Santa Barbara%
}

\begin{document}

\newtheorem{assumption}{Assumption}
\RestyleAlgo{ruled}

\maketitle

\begin{abstract}%
 We consider a safe optimization problem with bandit feedback in which an agent sequentially chooses actions and observes responses from the environment, with the goal of maximizing an arbitrary function of the response while respecting stage-wise constraints.
 We propose an algorithm for this problem, and study how the geometric properties of the constraint set impact the regret of the algorithm.
 In order to do so, we introduce the notion of the \emph{sharpness} of a particular constraint set, which characterizes the difficulty of performing learning within the constraint set in an uncertain setting. 
 This concept of sharpness allows us to identify the class of constraint sets for which the proposed algorithm is guaranteed to enjoy sublinear regret.
 Simulation results for this algorithm support the sublinear regret bound and provide empirical evidence that the sharpness of the constraint set impacts the performance of the algorithm.
\end{abstract}


\section{Introduction}

As contemporary learning and control paradigms expand into domains with stringent safety requirements, the need for control mechanisms that can provide such safety guarantees has grown significantly.
This has resulted in a plethora of novel safe learning problems in the literature through the lens of model predictive control~\cite{koller2018learning,hewing2020learning,chen2022robust}, reinforcement learning~\cite{junges2016safety,garcia2015comprehensive}, optimization~\cite{usmanova2019safe,fereydounian2020safe}, bandits~\cite{sui2015safe}, and many others.
Such problems are well suited for applications in which the control algorithm interacts with humans, which introduces uncertainties that need to be considered to ensure safety (e.g., clinical trials~\cite{sui2017correlational}, robotic systems \cite{berkenkamp2021bayesian}, and resource allocation in societal infrastructure through pricing~\cite{hutchinson2022safe}).

In this work, we are interested in a sequential decision making problem, where the decisions must be within an arbitrary and unknown compact safety set. We consider a safe optimization framework with bandit feedback, where the reward and the constraint set are known non-linear functions of the matrix multiplication of the action with an unknown parameter.
Compared to the existing literature, this problem is uniquely challenging because (1) both the decisions and the feedback from the environment are vectors, (2) the reward is an arbitrary function of the decision vector, and (3) the safety constraint on the decision vector is an arbitrary compact set. These challenges are however warranted, given that problems of this form appear in many real-world applications. For example, power flow constraints are nonlinear and nonconvex in general (\cite{molzahn2017survey}) and often solved with (nonlinear) convex relaxations (e.g. \cite{bai2008semidefinite,farivar2013branch}). 

We handle the challenge general safety sets present by introducing a geometric property of a set, which we call \emph{sharpness}, that is related to how difficult it is to perform learning within a particular safety set. This allows us to characterize the performance of our learning algorithm, measured in terms of regret, as a function of the sharpness of the safety set. Accordingly, we identify the class of safety sets (that includes all convex sets) for which we can establish a sublinear regret bound.

\textbf{Related work: }Sequential decision making under uncertainty with safety constraints has been an increasingly popular area of research among scholars. 
In particular, there have been various works that study optimization problems with uncertain constraint functions.
Depending on the specific problem setting, the constraint function is assumed to be linear \cite{usmanova2019safe,chaudhary2022safe, fereydounian2020safe}, have a Gaussian process prior \cite{sui2015safe,sui2018stagewise, berkenkamp2021bayesian} or be generally Lipschitz \cite{usmanova2020safe}.
In all of these works, the constraint is specified as requiring that the output of the (unknown) function is in the nonpositive orthant (i.e. $g(x) \leq 0$), whereas we model the constraint as requiring that the output of the unknown function is in some arbitrary set (i.e. $g(x) \in \mathcal{E}$ for some arbitrary $\mathcal{E}$).
This model warrants a unique analysis approach where we study how the geometry of this arbitrary constraint set impacts the performance of our algorithm.

Uncertain constraints have also been studied in the stochastic linear bandit setting.
In the conventional stochastic linear bandit setting (without uncertain constraints), an agent chooses an action vector at each time step and then receives a reward that is linear in expectation with respect to the action, with the aim of maximizing her cumulative reward (see e.g. \cite{dani2008stochastic,abbasi2011improved}).
One type of stochastic linear bandit problem with uncertain constraints is so-called conservative linear bandits \cite{kazerouni2017conservative,moradipari2020stage}, where the expected reward at each round needs to be close to a baseline reward.
Others consider a setting where there is an auxiliary constraint function.
Specifically, in \cite{amani2019linear} the constraint function depends on the (linearly transformed) reward parameter, while in \cite{pacchiano2021stochastic,moradipari2021safe, wang2022best} the constraint function is unrelated to the reward parameter and the learner receives noisy feedback of it.
Similar to stochastic linear bandits, we consider a problem where the expected response from the environment is a linear function of the action.
However, we take the response from the environment to be a vector rather than a scalar, consider the reward to be an arbitrary function of the expected response, and require the expected response from the environment to be within an arbitrary set.
The main novelty of this problem is the arbitrary constraint set, which necessitates new analysis techniques that might find broader applicability.

\textbf{Notation}:
For a vector $v \in \mathbb{R}^d$ and positive definite matrix $A \in \mathbb{R}^{d \times d}$, we denote the weighted 2-norm as $\| v \|_A = \sqrt{v^{\top} A v}$.
The minimum and maximum eigenvalues of a square matrix $M$ are denoted $\lambda_{min}(M)$ and $\lambda_{max} (M)$, respectively.
We denote the closed and open ball with radius $r$ and norm $\| \cdot \|$ as $\bar{\mathcal{B}}_{\| \cdot \|} (r)$ and $\mathcal{B}_{\| \cdot \|} (r)$, which are centered at the origin.
The condition number of a matrix $M$ is denoted as $\kappa(M)$.
For a set $\mathcal{D}$ and matrix $M$, we use the notation $M \mathcal{D} = \{M x : x \in \mathcal{D} \}$.
The set $\{1,2,...,n\}$ is denoted $[n]$.
We use $\Tilde{\mathcal{O}}$ to refer to big-O notation that ignores logarithmic factors.
We use the notation $\mathbf{e}_i$ to refer to a vector with $1$ as the $i$th position and $0$ everywhere else.
For a vector norm $\| \cdot \|$, the dual norm is denoted $\| \cdot \|_{\star}$.
The induced matrix $p$-norm is denoted $\| M \|_p$ for some matrix $M$, such that $\| M \|_p = \sup_{\| x \|_p = 1} \| A x \|_p$.
The Frobenius norm is denoted $\| M \|_F$ for some matrix $M$.

\section{Problem Setup}
We study a sequential decision-making problem where, in each round, an \emph{agent} chooses an \emph{action} and then the environment chooses a \emph{response} according to the action taken. 
Similar to stochastic linear bandits, we assume that the response from the environment is an unknown linear function of the action and that the agent observes the output of this linear function plus some noise.
However, our problem differs from stochastic linear bandits in that the response is multi-dimensional and the agent's goal is to accumulate \emph{reward}, which is an arbitrary known function of the response.
In our problem, the agent also needs to ensure that the response from the environment lies within a \emph{safety set} every round.

The details of the problem are described as follows. In each round $t \in [T]$, an agent chooses an action $x_t$ from the compact action set $\mathcal{A} \subset \mathbb{R}^d$ and then observes the noisy response $y_t := \Theta_* x_t + \epsilon_t$.
The matrix $\Theta_* \in \mathbb{R}^{n \times d}$ is an unknown parameter that is full rank, $\epsilon_t \in \mathbb{R}^n$ is random noise, and we have that $n \leq d$.
Upon choosing an action, the agent earns the reward $f(\Theta_* x_t)$, where the reward function $f:\mathbb{R}^n\rightarrow \mathbb{R}$ is known.
The agent needs to ensure that when it chooses actions, the response $\Theta_* x_t$ lies within the known compact safety set $\mathcal{E} \subset \mathbb{R}^n$ that has nonempty interior, or equivalently, that $x_t$ is in the \emph{unknown} feasible action set $\mathcal{X} := \{ x \in \mathcal{A} : \Theta_* x \in \mathcal{E} \}$.

With the actions that it chooses, the agent aims to maximize the cumulative reward that it achieves while ensuring that the safety constraint is satisfied for all rounds.
Therefore, the performance of the agent   can be measured with the cumulative regret, $R_T := \sum_{t=1}^T \left( f(\Theta_* x_*) - f(\Theta_* x_t) \right)$ where $x_* \in \argmax_{x \in \mathcal{\mathcal{X}}} f(\Theta_* x)$.

In the following two assumptions, we assume that the unknown parameter and feasible actions are bounded, and that the noise is subgaussian.
These assumptions are standard in related literature (e.g. \cite{abbasi2011improved,pacchiano2021stochastic}).

\begin{assumption}
\label{ass:bounds}
    For all $x$ in $\mathcal{A}$, there exists a constant $L$ such that $\| x \|_2 \leq L$. Additionally, there exists constant $S$ such that $\| \theta_*^i \|_2 \leq S$ for all $i \in [n]$, where $\theta_*^i$ is the $i$th row of $\Theta_*$. The constants $S$ and $L$ are known to the agent.
\end{assumption}

\begin{assumption}
\label{ass:noise}
    For all $t \in [T]$, the noise $\epsilon_t$ is element-wise conditionally $R$-subgaussian, such that given the history $\mathcal{F}_t = \sigma(x_1,x_2,...,x_{t+1},\epsilon_1, \epsilon_2, ...,\epsilon_t)$ and denoting the $i$th element of $\epsilon_t$ as $\epsilon_t^i$, it holds for all $i \in [n]$ that  $\mathbb{E}[\epsilon_t^i | \mathcal{F}_{t-1} ] = 0$ and $\mathbb{E}[e^{\lambda \epsilon_t^i} | \mathcal{F}_{t-1} ] \leq \exp(\frac{\lambda^2 R^2}{2}), \forall \lambda \in \mathbb{R}$.
    The constant $R$ is known to the agent.
\end{assumption}

We additionally assume that the reward function is Lipschitz.

\begin{assumption}
\label{ass:lipsch}
   $f$ is $M$-Lipschitz on $\mathcal{E}$ such that $| f(x) - f(y) | \leq M \| x - y \|_2$ for all $x,y$ in $\mathcal{E}$.
\end{assumption}

Lastly, we make an assumption which ensures that the knowledge provided to the agent by Assumption \ref{ass:bounds} is enough to choose initial actions that are safe.
Since it is known that $\Theta_*$ is in $\mathcal{C}^0 = \{ [\theta^1\ \theta^2\ ...\ \theta^n]^\top \in \mathbb{R}^{n \times d} : \| \theta^i \|_2 \leq S, \forall i \in [n] \}$ due to Assumption \ref{ass:bounds}, then it is also known that $\mathcal{G}^0 := \{ x \in \mathcal{A} : \Theta x \in \mathcal{E}, \forall \Theta \in \mathcal{C}^0 \}$ is a subset of $\mathcal{X}$.
Therefore, we ensure that the agent can initially choose safe actions by assuming that the interior of $\mathcal{G}^0$ is nonempty.
\begin{assumption}
\label{ass:init_set}
    The initial feasible set $\mathcal{G}^0$ has a nonempty interior.
\end{assumption}
We provide an algorithm for the stated problem in the next section.

\section{Proposed Algorithm}

We propose an algorithm to address the stated problem that operates by first performing pure exploration for an appropriate duration $T'$, as specified in the analysis, and then performing exploration-exploitation for the remaining rounds.
\IncMargin{1em}
\begin{algorithm2e}[t]
\caption{}
\label{alg:main_alg}
\DontPrintSemicolon
\LinesNumbered
\KwIn{$\mathcal{A},\mathcal{E},f,L, S$}
% \KwOut{$y$, the net activation}
\tcp{Pure Exploration}
\For{$t=1$ \KwTo $T'$}{
    Choose $x_t$ by randomly sampling $\mathcal{G}^0$, and observe response $y_t$.\;
}
Construct $\mathcal{C}_{T'}$ and $\mathcal{G}_{T'}$ with \eqref{eqn:conf_set} and \eqref{eqn:safe_act} respectively.\;
\tcp{Exploration-Exploitation}
\For{$t=T'+1$ \KwTo $T$}{
    Choose some $(x_t, \tilde{\Theta}_t) \in \argmax_{(x,\Theta) \in \mathcal{G}_{t-1} \times \mathcal{C}_{t-1}} f(\Theta x)$, and observe response $y_t$.\;
    Update $\mathcal{C}_{t}$ and $\mathcal{G}_{t}$ with \eqref{eqn:conf_set} and \eqref{eqn:safe_act} respectively.\;
}
\end{algorithm2e}
\DecMargin{1em}
The algorithm is given in Algorithm \ref{alg:main_alg}.

The pure exploration phase proceeds by randomly sampling actions from $\mathcal{G}^0$ such that $\lambda_- := \lambda_{\mathrm{min}} \left( \mathbb{E} \left[ x_t x_t^\top \right] \right) > 0$ for $t \in [T']$.
Such a scheme is possible given that $\mathcal{G}^0$ has a nonempty interior, although we leave the specific choice of sampling scheme as a design decision.\footnote{We give an example of a sampling scheme with $\lambda_- > 0$. Since $\mathcal{G}^0$ has a nonempty interior, there exists $v \in \mathbb{R}^d$ and $r > 0$ such that the open ball $v+\mathcal{B}_{2}(r)$ is a subset of $\mathcal{G}^0$. It follows that the closed ball $v+\bar{\mathcal{B}}_{2}(r/2)$ is a subset of $\mathcal{G}^0$. Therefore, one possible sampling scheme is to uniformly sample $u_t$ from the unit sphere i.i.d. such that $\mathbb{E} [ u_t u_t^\top ] = \frac{1}{d} I$, and then play $x_t = v+\frac{r}{2} u_t$. Therefore, $\mathbb{E} [ x_t x_t^\top ] = \mathbb{E} [ v v^\top ] + \frac{r}{2} \mathbb{E} [ v u_t^\top + u_t v^\top ] + \frac{r^2}{4} \mathbb{E} [ u_t u_t^\top ] 
 = v v^\top + \frac{r^2}{4 d} I$ given that $v$ is fixed, and it follows that $\lambda_- = \frac{r^2}{4 d} > 0$.}

Each round in the exploration-exploitation phase, $t \in (T',T]$, consists of first identifying the set of actions which will ensure safety given the current knowledge, and then choosing the optimistic action within this safe action set.
In order to both identify which actions are safe and to choose actions optimistically, we use confidence sets in which the parameter $\Theta_*$ lies with high probability.
Let $\theta^i_*$ be the $i$th row of $\Theta_*$, such that $\Theta_* = [\theta^1_*\ \theta^2_*\ ...\ \theta^n_*]^\top$, and let $y_t^i$ be the $i$th element of $y_t$, such that $y_t = [y_t^1\ y_t^2\ ...\ y_t^n]^{\top}$.
Then the regularized least-squares estimator of each $\theta^i_*$ is given by $\hat{\theta}^i_t = [V_t]^{-1} \sum_{s=1}^t x_s y_s^i$ at round $t$, where the gram matrix is $V_t = \nu I + \sum_{s=1}^t x_s \left[x_s\right]^\top$.
Using the regularized least-squares estimator for each row of $\Theta_*$, we define the confidence set in the following theorem from \cite{abbasi2011improved}.
\begin{theorem}
\label{thm:conf_set}
(Theorem 2 in \cite{abbasi2011improved}) Let Assumptions \ref{ass:bounds} and \ref{ass:noise} hold. Then if $x_t$ is in $\mathcal{A}$ for all $t$, we have with probability at least $1 - \delta$ that $\Theta_*$ lies in the set
\begin{equation}
\label{eqn:conf_set}
\mathcal{C}_t = \left\{ [\theta^1\ \theta^2\ ...\ \theta^n]^\top \in \mathbb{R}^{n \times d} : \left\| \theta^i - \hat{\theta}_t^i \right\|_{V_t} \leq \sqrt{\beta_t}, \forall i \in [n] \right\}
\end{equation}
for all $t \geq 0$, where
\begin{equation*}
    \sqrt{\beta_t} = R \sqrt{d \log \Big( \frac{1 + t L^2 / \nu}{\delta/n} \Big) } + \sqrt{\nu} S.
\end{equation*}
\end{theorem}
Using this set, we define a \emph{conservative inner approximation} of the feasible action set ($\mathcal{X}$) as
\begin{equation}
\label{eqn:safe_act}
    \mathcal{G}_t := \{ x \in \mathcal{A} : \Theta x \in \mathcal{E}, \forall \Theta \in \mathcal{C}_t \}.
\end{equation}
Note that the sets $\mathcal{C}_t$ and $\mathcal{G}_t$ are updated at the end of each round such that the agent has access to $\mathcal{C}_{t-1}$ and $\mathcal{G}_{t-1}$ in round $t$.
From the definition of $\mathcal{G}_t$, we can see that for any  $\Theta \in \mathcal{C}_{t}$ and $x \in \mathcal{G}_{t}$, it is guaranteed that $\Theta x$ is in the safety set $\mathcal{E}$.
Theorem \ref{thm:conf_set} states that $\Theta_*$ is in $\mathcal{C}_{t}$ for all rounds with high probability, so if the algorithm chooses $x_t$ from $\mathcal{G}_{t-1}$, the responses from the environment $\{ \Theta_* x_t \}_{\forall t \in (T',T]}$ will all be in $\mathcal{E}$ with high probability, and as such, they would ensure safety.

In order to choose these actions from the conservative action sets $\{ \mathcal{G}_{t-1} \}_{\forall t \in (T',T]}$ such that the regret is favorable, the algorithm behaves \emph{optimistically}.
That is, the algorithm chooses the best action in $\mathcal{G}_{t-1}$ assuming that the true parameter $\Theta_*$ is as favorable as possible given available information.
Since the agent knows that $\Theta_*$ is highly likely to be in the confidence set $\mathcal{C}_{t-1}$ in round $t$, the algorithm behaves optimistically by finding an action in $\mathcal{G}_{t-1}$ and parameter in $\mathcal{C}_{t-1}$ that maximize the possible reward.
Accordingly, the algorithm chooses the action as
\begin{equation}
\label{eqn:optim_update}
    (x_t, \tilde{\Theta}_t) \in \argmax_{(x,\Theta) \in \mathcal{G}_{t-1} \times \mathcal{C}_{t-1}} f(\Theta x),
\end{equation}
for every round $t \in (T',T]$.

It is important to recognize that, because $\mathcal{G}_{t-1}$ is a conservative inner approximation of $\mathcal{X}$, the optimal action $x_*$ may not be in $\mathcal{G}_{t-1}$.
% Hence, when the agent restricts itself to actions in $\mathcal{G}_t$ to ensure safety, 
Hence, how well $\mathcal{G}_{t-1}$ approximates $\mathcal{X}$ has an impact on how far $x_t$ is from the optimal action $x_*$, and hence how large the gap is between $f(\Theta_* x_*)$ and $f(\Theta_* x_t)$ (given the Lipschitz assumption on $f$).
The tightness with which $\mathcal{G}_t$ approximates $\mathcal{X}$ is evidently impacted by the size of $\mathcal{C}_t$, but as we show in the following section, the geometric properties of the safety constraint $\mathcal{E}$ and the action set $\mathcal{A}$ play a significant role as well.
 
\section{Regret Analysis}

In this section, we prove an upper bound on the cumulative regret of Algorithm \ref{alg:main_alg}.
A key aspect of the problem that impacts the regret is the geometric properties of the safety set and the action set.
As of now, we have not made any assumptions on these set.
However, we will show that the geometric properties of these sets will determine whether we can prove that the algorithm has sublinear regret.
To aid in this analysis, we introduce a geometric property of sets that we refer to as \emph{sharpness}, which plays a key role in the regret bound of the proposed algorithm.

\subsection{Geometric Properties of Safety Sets}
\label{sec:geom_props}

When the agent chooses an action, there is uncertainty as to what the response will be, necessitating the use of a conservative inner approximation of the set of safe actions.
Choosing actions from this inner approximation maintains safety because it ensures that every reasonably possible response to the chosen action (i.e. every $\Theta x_t$ for $\Theta \in \mathcal{C}_{t-1}$) satisfies the safety constraint.
In essence, this ensures that some region around the expected response lies within the safety constraint.
One can imagine that this region will not ``fit" well in to any ``sharp" corners that the safety set may have, and hence the inner approximation will be looser for safety sets with ``sharp" corners, resulting in less favorable regret.
We formalize this notion of \emph{sharpness} in the following series of definitions.
The proofs from this section are given in Appendix \ref{apx:geom_proofs}.

In order to study the impact that a safety set's geometry has on the tightness of the conservative inner approximation, we first present a more general type of inner approximation that we call the \emph{shrunk version of a set}.
Similar to how the conservative inner approximation ensures that the set of all reasonably possible responses are within the safety constraint, the shrunk version of a set ensures that a closed ball at each point is within the original set.
This is formally defined in the following definition, which uses the closed ball, $\bar{\mathcal{B}}_{\| \cdot \|} (r) := \{ x \in \mathbb{R}^m : \| x \| \leq r \}$, where $r$ is the radius and the particular norm is $\| \cdot \|$.
\begin{definition}
\label{def:shrunk_set}
    For a compact set $\mathcal{D} \subset \mathbb{R}^m$, a norm $\| \cdot \|$, and a nonnegative scalar $\Delta$, we define the \emph{shrunk version} of $\mathcal{D}$ as $\mathcal{D}_{\Delta}^{\| \cdot \|} := \{ x \in \mathcal{D} : x + v \in \mathcal{D}, \forall v \in \bar{\mathcal{B}}_{\| \cdot \|} (\Delta) \}$.\footnote{We can equivalently define $\mathcal{D}_{\Delta}^{\| \cdot \|}$ using Minkowski subtraction.
    The Minkowski subtraction of sets $A,B \subseteq \mathbb{R}^m$ is defined as $A \ominus B := \{a - b : a \in A, b \in B \}$, or equivalently, $A \ominus B = \{ x \in \mathbb{R}^m : x + B \subseteq A \}$ (\cite{schneider2014convex}).
    Therefore, we can write that $\mathcal{D}_{\Delta}^{\| \cdot \|} = \mathcal{D} \ominus \mathcal{B}_{\| \cdot \|} (\Delta)$ for $\Delta \geq 0$.}
\end{definition}
Given the above definition of the shrunk version of a set, one can consider the maximum shrinkage that a set can withstand while still being nonempty.
We introduce the \emph{maximum shrinkage of a set} in the following definition.
\begin{definition}
\label{def:delt_max}
    For a compact set $\mathcal{D} \subset \mathbb{R}^m$ and a norm $\| \cdot \|$, we define the \emph{maximum shrinkage} of $\mathcal{D}$, as $H_{\mathcal{D}}^{\| \cdot \|} := \sup\{ \Delta : \mathcal{D}_{\Delta}^{\| \cdot \|} \neq \emptyset \}$.
\end{definition}
We can now formally define the \emph{sharpness of a set} as the maximum  distance from any point in a set to the nearest point in the shrunk version of that set.
\begin{definition}
For a compact set $\mathcal{D} \subset \mathbb{R}^m$ and norm $\| \cdot \|$, we define the \emph{sharpness} of $\mathcal{D}$ as 
\begin{equation*}
    \mathrm{Sharp}_{\mathcal{D}}^{\| \cdot \|} (\Delta) := \sup_{x \in \mathcal{D}} \inf_{y \in \mathcal{D}_{\Delta}^{\| \cdot \|}}  \left \| y - x \right \|_2,
\end{equation*}
for all non-negative $\Delta$ such that $\mathcal{D}_{\Delta}^{\| \cdot \|}$ is nonempty.\footnote{Sharpness can also be define with the Hausdorff metric between sets (see \cite{schneider2014convex} Sec. 1.8) such that $\mathrm{Sharp}_{\mathcal{D}}^{\| \cdot \|} (\Delta) = d_H (\mathcal{D}, \mathcal{D}_{\Delta}^{\| \cdot \|})$.}
\end{definition}
Sharpness is applicable to the analysis of safe learning algorithms because it upper bounds how far an optimal point within the safe set (e.g. some set $\mathcal{D}$) is from a conservative inner approximation of that safe set (e.g. $\mathcal{D}_{\Delta}^{\| \cdot \|}$ or a superset of $\mathcal{D}_{\Delta}^{\| \cdot \|}$).
\begin{wrapfigure}{R}{0.45\textwidth}
    % \vspace{0.1in}
        \centering
    \includegraphics[width=0.45\textwidth]{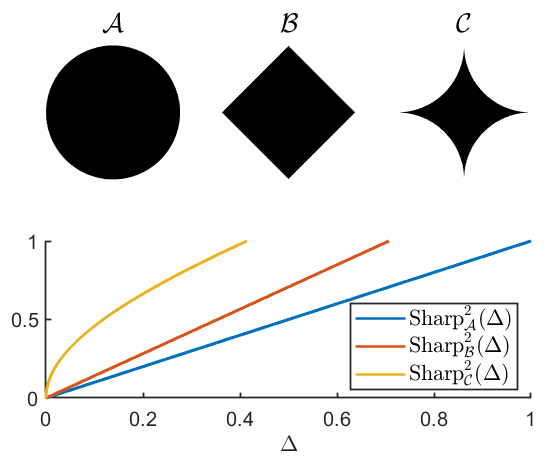}
    \caption{The 2-norm sharpness of three different sets in $\mathbb{R}^2$.}
    \label{fig:sharp_examps}
    % \vspace{-0.6in}
\end{wrapfigure}
To demonstrate how the geometry of a set impacts its sharpness, the sharpness of several different sets in $\mathbb{R}^2$ is plotted in Figure \ref{fig:sharp_examps}.
One can see that sets with ``sharper" corners have greater sharpness for the same value of $\Delta$.
Also note that we use $\mathcal{D}_{\Delta}^{p}$, $H_{\mathcal{D}}^p$ and $\mathrm{Sharp}_{\mathcal{D}}^{p} (\Delta)$ to refer to the shrunk set, maximum shrinkage and sharpness of some set $\mathcal{D}$ with respect to the $p$-norm.

We now show some simple properties related to when the shrunk version of a set is nonempty and therefore when the sharpness is defined.
First, we have that the shrunk version of a compact set is nonempty for some positive shrinkage precisely when the set has a nonempty interior.
\begin{proposition}
\label{prop:no_empty}
    For a compact set $\mathcal{D} \subset \mathbb{R}^m$, there exists a $\Delta > 0$ such that $\mathcal{D}_{\Delta}^{\| \cdot \|}$ is nonempty if and only if $\mathcal{D}$ has a nonempty interior.
\end{proposition}
Next, we show that the shrunk version of a compact set with nonempty interior is nonempty for all shrinkage less than or equal to the maximum shrinkage of the set.
This indicates that sharpness is defined on the closed interval from zero to the maximum shrinkage.
\begin{proposition}
\label{prop:range}
    For a compact set $\mathcal{D} \subset \mathbb{R}^m$ with nonempty interior, we have that $\mathcal{D}_{\Delta}^{\| \cdot \|}$ is nonempty for all $\Delta \in [0,H_{\mathcal{D}}^{\| \cdot \|}]$.
\end{proposition}

For the remainder of this section, we will study the sharpness of different types of compact sets with nonempty interiors.
The first type of set that we study is the polytope, 
which is the convex hull of a finite set of points, or equivalently, the bounded intersection of a finite number of closed half-spaces.
Polytopes capture a wide variety of constraints in the real world and are frequently used for safety sets in safe learning  (e.g., \cite{chaudhary2022safe,fereydounian2020safe,usmanova2019safe}).
We use the polyhedron representation of a polytope,  $\mathcal{D} = \{ x \in \mathbb{R}^m : Ax \leq b\}$ with $A \in \mathbb{R}^{p \times m}$ and $b \in \mathbb{R}^p$, where there are no redundant constraints.
We define $a_j \in \mathbb{R}^m$ as the $j$th row of $A$ such that $A = [a_1\ a_2\ ...\ a_p]^\top$ and $b_j \in \mathbb{R}$ as the $j$th element of $b$ such that $b = [b_1\ b_2\ ...\ b_p]^\top$.
We also use $\mathcal{I}_A$ to refer to the collection of all sets of $m$ indices such that for each $\{i_1, i_2, ..., i_m\} \in \mathcal{I}_A$ the vectors $a_{i_1}, a_{i_2}, ..., a_{i_m}$ are linearly independent.
For each $\ell \in \mathcal{I}_A$ where $\ell = \{i_1, i_2, ..., i_m\}$, we write $A^\ell = [a_{i_1}\  a_{i_2}\ ...\ a_{i_m}]^\top$ and denote its condition number by $\kappa(A^\ell)$.
Using this notation, the following proposition shows that the sharpness of a polytope is bounded by a function that is linear in shrinkage. 
\begin{proposition}
\label{prop:sharp_polyt}
    For a polytope $\mathcal{D} = \{ x \in \mathbb{R}^m : Ax \leq b\}$ with nonempty interior, we have that $\mathrm{Sharp}_{\mathcal{D}}^{\| \cdot \|} (\Delta) \leq \sqrt{m} C_{\| \cdot \|} K_{\mathcal{D}} \Delta$, where $K_{\mathcal{D}} := \max_{\ell \in \mathcal{I}_A} \kappa(A^\ell)$ and $C_{\| \cdot \|} := \max_{\| y \| = 1} \| y \|_2$.
\end{proposition}
The sharpness bound in Proposition \ref{prop:sharp_polyt} is proportional to the constant $K_{\mathcal{D}}$, which is the maximum condition number of all sets of $m$ linearly independent constraints.
Since there are $m$ linearly independent constraints that are active at each vertex, $K_{\mathcal{D}}$ upper bounds the condition number of the active constraints at each vertex.
This is an intuitive measure of the sharpness of a polytope, given that the condition number of the constraints indicate how close to parallel they are.
Also, note that the term $C_{\| \cdot \|}$ in Proposition \ref{prop:sharp_polyt} may depend on the dimension.
For example, when the infinity-norm is used, $C_{\| \cdot \|_{\infty}} = \sqrt{m}$, and when the 1-norm is used, $C_{\| \cdot \|_1}~=~1$.

Using the sharpness bound that we developed for polytopes, we can study more general sets.
The key intuition that we use to study more general sets is that we can define subsets of the original set which, with appropriate construction, bound the original set in terms of sharpness.
In particular, we construct polytopic subsets of the original set in order to provide sharpness bounds that are linear with respect to shrinkage.
Being able establish linear bounds on the sharpness is important because it allows us to establish sublinear regret bounds on the proposed algorithm, which we discuss in the next section.
In order to develop a bound that uses polytopic subsets, we define the families of polytopes that can be used to bound the sharpness of a given set.
\begin{definition}
    For a point $x$ in the compact set $\mathcal{D} \subset \mathbb{R}^m$, we define $F_{\mathcal{D}}(x)$ as the family of polytopes with nonempty interior that contain $x$ and are subsets of $\mathcal{D}$.
\end{definition}
From this, we define the class of sets for which we can use polytopic subsets to bound the sharpness.
\begin{definition}
    A compact set $\mathcal{D} \subset \mathbb{R}^m$ is referred to as \emph{polytope-sharp} if $F_{\mathcal{D}}(x)$ is nonempty for all $x \in \mathcal{D}$.
\end{definition}
We can see that the class of polytope-sharp sets are those for which a collection of polytopes can be constructed to contain each point in the set while still being subsets of the original set.
\begin{figure}
    \centering
    \subfigure{
        \centering
        \includegraphics[width=5cm]{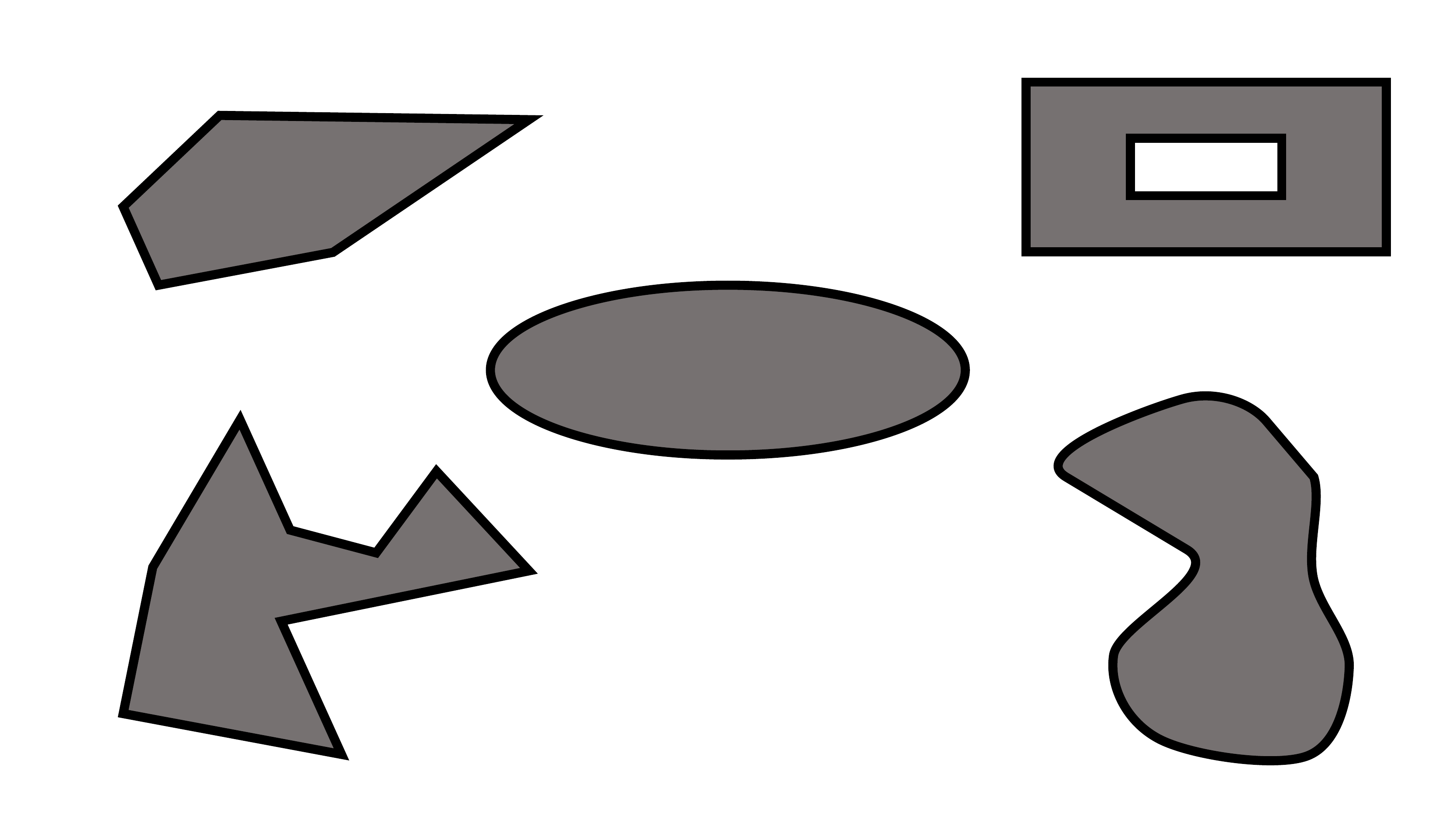}
        }
        \hspace{0.75in}
    \subfigure{
        \includegraphics[width=5cm]{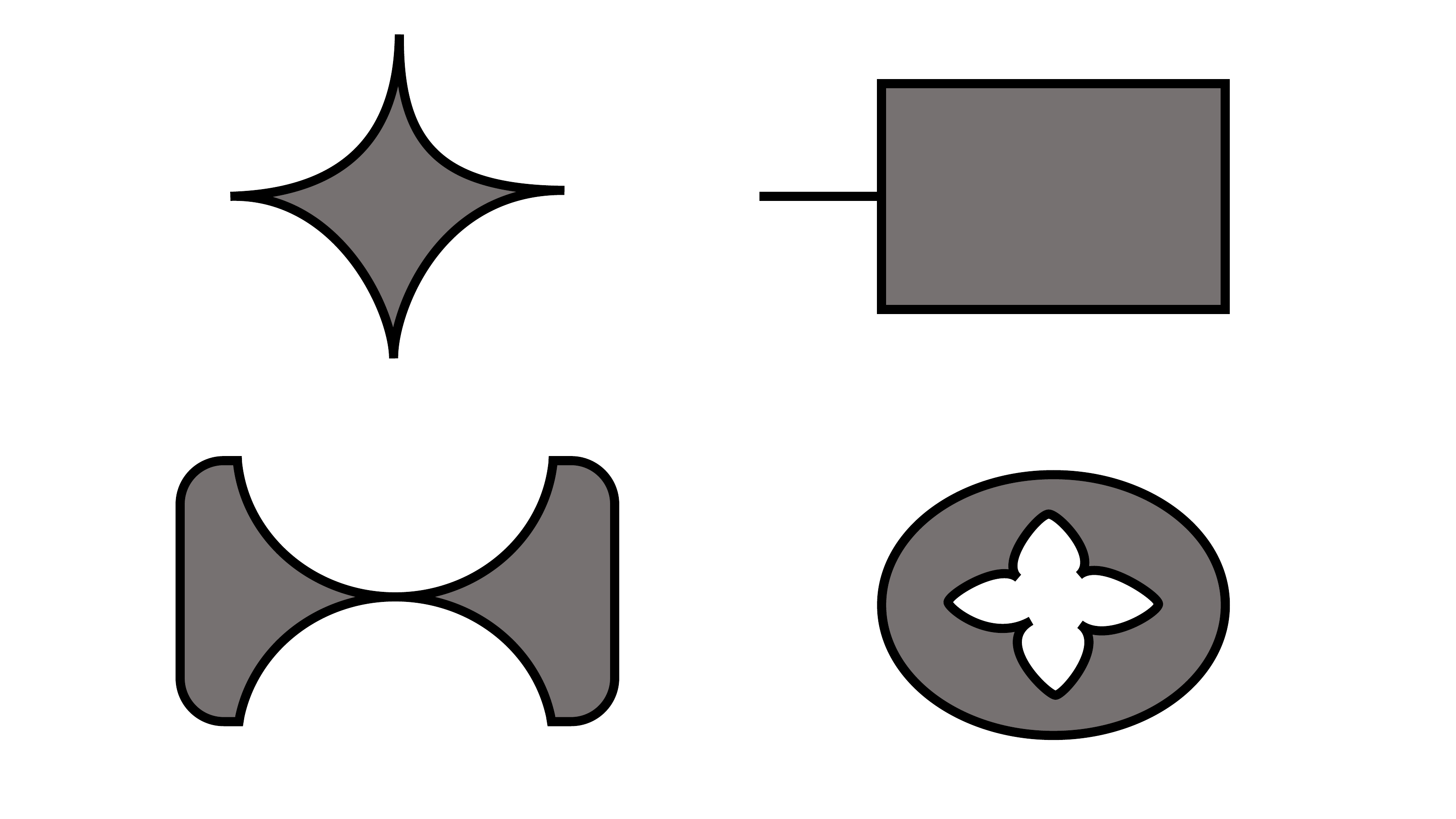}
        }
    \caption{Examples of sets in $\mathbb{R}^2$ which are \emph{polytope-sharp} and hence can be bounded linearly with respect to shrinkage via Proposition \ref{prop:linear_sets} (a), and sets that are not \emph{polytope-sharp} (b).}
    \label{fig:linear_sharp}
    \vspace{-0.2in}
\end{figure}
Examples of sets that meet this criterion and sets that do not meet this criterion are illustrated in Figure \ref{fig:linear_sharp}.
With this definition, we can then present the following proposition, which provides a linear sharpness bound on sets that are polytope-sharp.
\begin{proposition}
\label{prop:linear_sets}
    Let $\mathcal{D} \subset \mathbb{R}^m$ be a compact set that is polytope-sharp, and choose some arbitrary $\mathcal{F}_x \in F_{\mathcal{D}}(x)$ for each $x \in \mathcal{D}$.
    Then, we have that 
    \begin{equation*}
        \mathrm{Sharp}_{\mathcal{D}}^{\| \cdot \|} (\Delta) \leq \sqrt{m} \bar{C}_{\| \cdot \|} \Gamma_{\mathcal{D}} \Delta,
    \end{equation*}
    where, $\bar{C}_{\| \cdot \|} = \max( C_{\| \cdot \|},1)$ and
    \begin{equation*}
        \Gamma_{\mathcal{D}} := \max \left \{ \bar{K}_{\mathcal{F}}, \frac{r_{\mathcal{D}}}{\bar{H}_{\mathcal{F}}^{\| \cdot \|}} \right \},
    \end{equation*}
    with $\bar{K}_{\mathcal{F}} := \max_{x \in \mathcal{D}} K_{\mathcal{F}_x}$, $\bar{H}_{\mathcal{F}}^{\| \cdot \|} := \min_{x \in \mathcal{D}} H_{\mathcal{F}_x}^{\| \cdot \|}$ and $r_{\mathcal{D}} := \max_{x,y \in \mathcal{D}} \| x - y \|_2$.
\end{proposition}
In addition to providing a linear sharpness bound on polytope-sharp sets, Proposition \ref{prop:linear_sets} also indicates that, when the polytopes are small, i.e. $\bar{H}_{\mathcal{F}}^{\| \cdot \|}$ is small, or the polytopes are sharp, i.e. $\bar{K}_{\mathcal{F}}$ is large, then the sharpness bound is larger and therefore less favorable.
From Proposition \ref{prop:linear_sets}, we can also immediately show that every compact, convex set with nonempty interior is linearly sharp.
\begin{corollary}
\label{cor:sharp_conv}
    If a compact set $\mathcal{D} \subset \mathbb{R}^m$ with nonempty interior is convex, then it is polytope-sharp and it holds that $\mathrm{Sharp}_{\mathcal{D}}^{\| \cdot \|} (\Delta) \leq \sqrt{m} \bar{C}_{\| \cdot \|} \Gamma_{\mathcal{D}} \Delta$.
\end{corollary}
It is important to note that although all compact convex sets are polytope-sharp, there are also nonconvex sets that are polytope-sharp.

\subsection{Regret Bound}

We will now use the work in the previous section to establish a sublinear bound on the cumulative regret of Algorithm \ref{alg:main_alg}.
In order to do so, we define the set of feasible responses as $\mathcal{Y} := \Theta_* \mathcal{A} \cap \mathcal{E}$, where we use the notation $\Theta_* \mathcal{A} = \{ \Theta_* x : x \in \mathcal{A} \}$.
The set $\mathcal{Y}$ reflects the set of responses that are possible given the action set $\mathcal{A}$ and the safety set $\mathcal{E}$.
The sharpness of $\mathcal{Y}$ is used in the regret bound for Algorithm \ref{alg:main_alg}, as shown in Theorem \ref{thm:main}.
Although one might expect that the sharpness of $\mathcal{E}$ would be in the regret bound (instead of the sharpness of $\mathcal{Y}$), the set $\mathcal{A}$ can impact the distance from the optimal action $x^*$ to the set $\mathcal{G}_{t-1}$ and hence it is insufficient to solely use the sharpness of $\mathcal{E}$ in the regret analysis. 
Therefore, we use the sharpness of $\mathcal{Y}$ to capture both the sharpness of $\mathcal{E}$ and any unfavorable effects due to the specific $\mathcal{A}$ in a particular problem setting.
\begin{theorem}
\label{thm:main}
    Let Assumptions \labelcref{ass:bounds,ass:init_set,ass:lipsch,ass:noise} hold. 
    With probability at least $1 - 2 \delta$, we have that the regret of Algorithm \ref{alg:main_alg} is bounded as
    \begin{equation*}
        \begin{split}
            R_T \leq {}& 2 M \sqrt{n} LS T' + M (T - T') \mathrm{Sharp}_{\mathcal{Y}}^{\infty} \left( \frac{2  \sqrt{2 \beta_T} L }{\sqrt{2 \nu + \lambda_- T'}} \right) \\
            & + M \max (H_{\mathcal{Y}}^\infty,1) \sqrt{n 8 \beta_T (T - T') d \log \left( \frac{1 + T L^2}{d \nu} \right) }.
        \end{split}
    \end{equation*}
    for any $T' \geq \max( t_{\delta}, t_h)$ where $t_\delta := \frac{8 L^2}{\lambda_-} \log(\frac{d}{\delta})$ and $t_h := \frac{8 \beta_T L^2}{\lambda_- \left( H_{\mathcal{Y}}^\infty \right)^2} - \frac{2 \nu}{\lambda_-}$.
\end{theorem}
\begin{corollary}
\label{cor:main}
    Assume the same as Theorem \ref{thm:main}.
    If $\mathcal{Y}$ is polytope-sharp, then the regret of Algorithm~\ref{alg:main_alg} satisfies
    \begin{equation*}
        \begin{split}
            R_T \leq {}& 2 M \sqrt{n} LS T' + \frac{2  n \sqrt{2 \beta_T} \Gamma_{\mathcal{Y}} L M (T - T') }{\sqrt{2 \nu + \lambda_- T'}}  \\
            & + M \max (H_{\mathcal{Y}}^\infty,1) \sqrt{n 8 \beta_T (T - T') d \log\left( \frac{1 + T L^2}{d \nu}\right)}
        \end{split}
    \end{equation*}    
    with probability at least $1 - 2 \delta$ when $T' \geq \max( t_{\delta}, t_h)$. In particular, choosing $T' = \max(T^{2/3},t_{\delta}, t_h)$ ensures that $R_T = \tilde{\mathcal{O}} \left(T^{2/3} \right)$.
\end{corollary}
We can see that the regret bound depends on the sharpness of $\mathcal{Y}$, and as shown in Corollary \ref{cor:main}, is $\tilde{\mathcal{O}}\left(T^{2/3} \right)$ when $\mathcal{Y}$ is polytope-sharp.
Note that the agent needs to know the maximum shrinkage of $\mathcal{Y}$, or a lower bound of it, in order to appropriately choose $T'$.
If there is a known subset of $\mathcal{Y}$ or it is known that $\mathcal{E}$ is a subset of $\Theta_* \mathcal{A}$ then the agent can calculate a lower bound on the maximum shrinkage of $\mathcal{Y}$.
Otherwise, there might be application specific information that provides a conservative estimate of the maximum shrinkage of $\mathcal{Y}$.

The complete proof of Theorem \ref{thm:main} is given in Appendix \ref{apx:thm_main}.
This proof utilizes a decomposition of the instantaneous regret given by
\begin{equation}
\label{eqn:reg_decomp}
    r_t := f(\Theta_* x_*) - f(\Theta_* x_t) = \underbrace{f(\Theta_* x_*) - f(\tilde{\Theta}_t x_t)}_{\text{Term I}} + \underbrace{f(\tilde{\Theta}_t x_t) - f(\Theta_* x_t)}_{\text{Term II}}.
\end{equation}
Term I captures the suboptimality of the optimistic pair $(x_t, \tilde{\Theta}_t )$ from \eqref{eqn:optim_update}, while Term II captures the shrinkage of the confidence set $\mathcal{C}_t$.
The pair $(x_t, \tilde{\Theta}_t )$ may be suboptimal due to the fact that $\mathcal{G}_t$ is a strict subset of $\mathcal{X}$, which is necessary to ensure safety.
Although the analysis of Term II can be handled with conventional bandit analysis, the analysis of Term I requires novel techniques, including sharpness, as we discuss in the following paragraph.

The bound on Term I is given in the following lemma.
\begin{lemma}
\label{lem:term_i}
    Let Assumptions \labelcref{ass:bounds,ass:init_set,ass:lipsch,ass:noise} hold. For $t \in (T', T]$, $\mathrm{Term\ I}$ is bounded as
    \begin{equation*}
        \mathrm{Term\ I} := f(\Theta_* x_*) - f(\tilde{\Theta}_t x_t) \leq M \mathrm{Sharp}_{\mathcal{Y}}^{\infty} \left( \frac{2  \sqrt{2 \beta_T} L }{\sqrt{2 \nu + \lambda_- T'}} \right)
    \end{equation*}
    when $T' \geq \max(t_\delta, t_h)$ with probability at least $1 - 2 \delta$.
\end{lemma}
The proof of Lemma \ref{lem:term_i} is given in Appendix \ref{sec:term_i} and considers a shrunk version of $\mathcal{Y}$ such that every possible $y$ in the shrunk version of $\mathcal{Y}$ can be given by $\Theta x$ with some $\Theta \in \mathcal{C}_t$ and some $x \in \mathcal{G}_t$.
This implies that $f(\tilde{\Theta} x_t)$ is greater than or equal to $f(y)$ for every $y$ in the shrunk version of $\mathcal{Y}$ and hence we can bound Term I with the difference between the optimal reward ($f(y_*)$, where $y_* = \Theta_* x_*$) and the reward from some $y$ in the shrunk version of $\mathcal{Y}$.
With the Lipschitz assumption on $f$, this can be bounded with the difference between $y_*$ and some $y$ in the shrunk version $\mathcal{Y}$.
By choosing $y$ to be the point in the shrunk version of $\mathcal{Y}$ that is closest to $y_*$, we can ultimately bound the regret with the sharpness of $\mathcal{Y}$ as given in Lemma \ref{lem:term_i}.

\begin{wrapfigure}{r}{0.5\textwidth}
\centering
    \vspace{-0.7in}
    \includegraphics[width=0.5\textwidth]{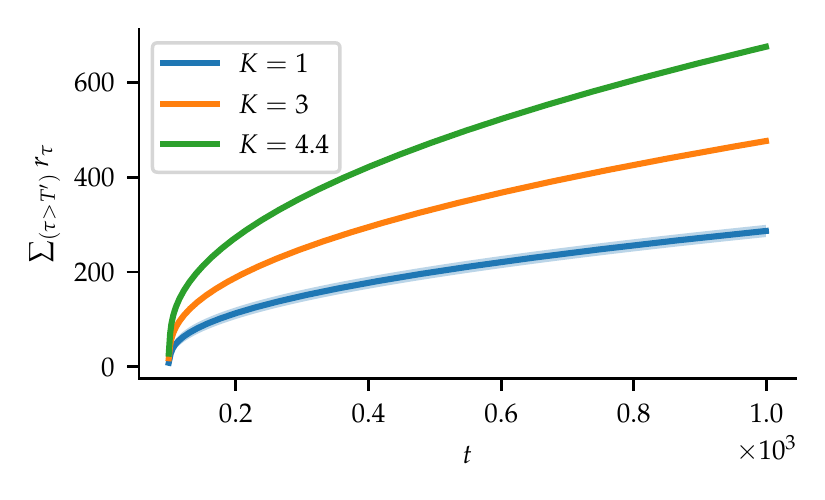}
    \vspace{-.4in}
    \caption{The cumulative sum of the instantaneous regret in the exploration-exploitation phase of the algorithm with polytopic constraint sets that have different $K$ constants (as defined in Proposition \ref{prop:sharp_polyt}).}
    \label{fig:cumreg}
    \vspace{-0.3in}
\end{wrapfigure}

\section{Numerical Experiments}
We simulated the results of the algorithm with three different polytopic safety sets of different sharpness in a problem setting where $n=3$ and $d=3$.
The cumulative sum of the instantaneous regret in the exploration-exploitation phase of the algorithm is shown in Figure \ref{fig:cumreg} for each polytopic safety set.
The solid line is an average of six trials and the shaded region indicates the 95\% confidence interval over the trials.
For each safety set, the plot shows its $K$ constant, as defined in Proposition \ref{prop:sharp_polyt}.
Each simulation has a different realization of the noise $\{ \epsilon_t \}_{t \in [T]}$.
Otherwise, the problem and algorithm parameters are the same for every simulation, and all the polytopic safety sets have the same maximum shrinkage.
Also, note that the action set $\mathcal{A}$ is chosen to be non-restrictive, such that $\mathcal{Y} = \mathcal{E}$.
Figure \ref{fig:cumreg} provides some empirical support for the sublinear regret bound in Theorem \ref{thm:main} and also indicates that the sharpness of the safe set impacts the regret of the algorithm.
The details of the simulation are given in Appendix \ref{apx:exps}.

\acks{This work was supported by NSF grant \#1847096.}

\bibliography{references.bib}

\newcommand{\noop}[1]{}
\begin{thebibliography}{27}
\providecommand{\natexlab}[1]{#1}
\providecommand{\url}[1]{\texttt{#1}}
\expandafter\ifx\csname urlstyle\endcsname\relax
  \providecommand{\doi}[1]{doi: #1}\else
  \providecommand{\doi}{doi: \begingroup \urlstyle{rm}\Url}\fi

\bibitem[Abbasi-Yadkori et~al.(2011)Abbasi-Yadkori, P{\'a}l, and
  Szepesv{\'a}ri]{abbasi2011improved}
Yasin Abbasi-Yadkori, D{\'a}vid P{\'a}l, and Csaba Szepesv{\'a}ri.
\newblock Improved algorithms for linear stochastic bandits.
\newblock \emph{Advances in neural information processing systems}, 24, 2011.

\bibitem[Amani et~al.(2019)Amani, Alizadeh, and Thrampoulidis]{amani2019linear}
Sanae Amani, Mahnoosh Alizadeh, and Christos Thrampoulidis.
\newblock Linear stochastic bandits under safety constraints.
\newblock \emph{Advances in Neural Information Processing Systems}, 32, 2019.

\bibitem[Bai et~al.(2008)Bai, Wei, Fujisawa, and Wang]{bai2008semidefinite}
Xiaoqing Bai, Hua Wei, Katsuki Fujisawa, and Yong Wang.
\newblock Semidefinite programming for optimal power flow problems.
\newblock \emph{International Journal of Electrical Power \& Energy Systems},
  30\penalty0 (6-7):\penalty0 383--392, 2008.

\bibitem[Berkenkamp et~al.(2021)Berkenkamp, Krause, and
  Schoellig]{berkenkamp2021bayesian}
Felix Berkenkamp, Andreas Krause, and Angela~P Schoellig.
\newblock Bayesian optimization with safety constraints: safe and automatic
  parameter tuning in robotics.
\newblock \emph{Machine Learning}, pages 1--35, 2021.

\bibitem[Boyd et~al.(2004)Boyd, Boyd, and Vandenberghe]{boyd2004convex}
Stephen Boyd, Stephen~P Boyd, and Lieven Vandenberghe.
\newblock \emph{Convex optimization}.
\newblock Cambridge university press, 2004.

\bibitem[Chaudhary and Kalathil(2022)]{chaudhary2022safe}
Sapana Chaudhary and Dileep Kalathil.
\newblock Safe online convex optimization with unknown linear safety
  constraints.
\newblock In \emph{Proceedings of the AAAI Conference on Artificial
  Intelligence}, volume~36, pages 6175--6182, 2022.

\bibitem[Chen et~al.(2022)Chen, Li, Preciado, and Matni]{chen2022robust}
Shaoru Chen, Ning-Yuan Li, Victor~M Preciado, and Nikolai Matni.
\newblock Robust model predictive control of time-delay systems through system
  level synthesis.
\newblock \emph{arXiv preprint arXiv:2209.11841}, 2022.

\bibitem[Dani et~al.(2008)Dani, Hayes, and Kakade]{dani2008stochastic}
Varsha Dani, Thomas~P Hayes, and Sham~M Kakade.
\newblock Stochastic linear optimization under bandit feedback.
\newblock 2008.

\bibitem[Farivar and Low(2013)]{farivar2013branch}
Masoud Farivar and Steven~H Low.
\newblock Branch flow model: Relaxations and convexification—part i.
\newblock \emph{IEEE Transactions on Power Systems}, 28\penalty0 (3):\penalty0
  2554--2564, 2013.

\bibitem[Fereydounian et~al.(2020)Fereydounian, Shen, Mokhtari, Karbasi, and
  Hassani]{fereydounian2020safe}
Mohammad Fereydounian, Zebang Shen, Aryan Mokhtari, Amin Karbasi, and Hamed
  Hassani.
\newblock Safe learning under uncertain objectives and constraints.
\newblock \emph{arXiv preprint arXiv:2006.13326}, 2020.

\bibitem[Garc{\i}a and Fern{\'a}ndez(2015)]{garcia2015comprehensive}
Javier Garc{\i}a and Fernando Fern{\'a}ndez.
\newblock A comprehensive survey on safe reinforcement learning.
\newblock \emph{Journal of Machine Learning Research}, 16\penalty0
  (1):\penalty0 1437--1480, 2015.

\bibitem[Hewing et~al.(2020)Hewing, Wabersich, Menner, and
  Zeilinger]{hewing2020learning}
Lukas Hewing, Kim~P Wabersich, Marcel Menner, and Melanie~N Zeilinger.
\newblock Learning-based model predictive control: Toward safe learning in
  control.
\newblock \emph{Annual Review of Control, Robotics, and Autonomous Systems},
  3:\penalty0 269--296, 2020.

\bibitem[Hutchinson et~al.(2022)Hutchinson, Turan, and
  Alizadeh]{hutchinson2022safe}
Spencer Hutchinson, Berkay Turan, and Mahnoosh Alizadeh.
\newblock A safe pricing mechanism for distributed resource allocation with
  bandit feedback.
\newblock In \emph{2022 IEEE 61st Conference on Decision and Control (CDC)},
  pages 5092--5098. IEEE, 2022.

\bibitem[Junges et~al.(2016)Junges, Jansen, Dehnert, Topcu, and
  Katoen]{junges2016safety}
Sebastian Junges, Nils Jansen, Christian Dehnert, Ufuk Topcu, and Joost-Pieter
  Katoen.
\newblock Safety-constrained reinforcement learning for mdps.
\newblock In \emph{International conference on tools and algorithms for the
  construction and analysis of systems}, pages 130--146. Springer, 2016.

\bibitem[Kazerouni et~al.(2017)Kazerouni, Ghavamzadeh, Abbasi~Yadkori, and
  Van~Roy]{kazerouni2017conservative}
Abbas Kazerouni, Mohammad Ghavamzadeh, Yasin Abbasi~Yadkori, and Benjamin
  Van~Roy.
\newblock Conservative contextual linear bandits.
\newblock \emph{Advances in Neural Information Processing Systems}, 30, 2017.

\bibitem[Koller et~al.(2018)Koller, Berkenkamp, Turchetta, and
  Krause]{koller2018learning}
Torsten Koller, Felix Berkenkamp, Matteo Turchetta, and Andreas Krause.
\newblock Learning-based model predictive control for safe exploration.
\newblock In \emph{2018 IEEE conference on decision and control (CDC)}, pages
  6059--6066. IEEE, 2018.

\bibitem[Molzahn et~al.(2017)Molzahn, D{\"o}rfler, Sandberg, Low, Chakrabarti,
  Baldick, and Lavaei]{molzahn2017survey}
Daniel~K Molzahn, Florian D{\"o}rfler, Henrik Sandberg, Steven~H Low, Sambuddha
  Chakrabarti, Ross Baldick, and Javad Lavaei.
\newblock A survey of distributed optimization and control algorithms for
  electric power systems.
\newblock \emph{IEEE Transactions on Smart Grid}, 8\penalty0 (6):\penalty0
  2941--2962, 2017.

\bibitem[Moradipari et~al.(2020)Moradipari, Thrampoulidis, and
  Alizadeh]{moradipari2020stage}
Ahmadreza Moradipari, Christos Thrampoulidis, and Mahnoosh Alizadeh.
\newblock Stage-wise conservative linear bandits.
\newblock \emph{Advances in neural information processing systems},
  33:\penalty0 11191--11201, 2020.

\bibitem[Moradipari et~al.(2021)Moradipari, Amani, Alizadeh, and
  Thrampoulidis]{moradipari2021safe}
Ahmadreza Moradipari, Sanae Amani, Mahnoosh Alizadeh, and Christos
  Thrampoulidis.
\newblock Safe linear thompson sampling with side information.
\newblock \emph{IEEE Transactions on Signal Processing}, 69:\penalty0
  3755--3767, 2021.

\bibitem[Pacchiano et~al.(2021)Pacchiano, Ghavamzadeh, Bartlett, and
  Jiang]{pacchiano2021stochastic}
Aldo Pacchiano, Mohammad Ghavamzadeh, Peter Bartlett, and Heinrich Jiang.
\newblock Stochastic bandits with linear constraints.
\newblock In \emph{International Conference on Artificial Intelligence and
  Statistics}, pages 2827--2835. PMLR, 2021.

\bibitem[Schneider(2014)]{schneider2014convex}
Rolf Schneider.
\newblock \emph{Convex bodies: the Brunn--Minkowski theory}.
\newblock Number 151. Cambridge university press, 2014.

\bibitem[Sui and Burdick(2017)]{sui2017correlational}
Yanan Sui and Joel~W Burdick.
\newblock Correlational dueling bandits with application to clinical treatment
  in large decision spaces.
\newblock In \emph{Proceedings of the 26th International Joint Conference on
  Artificial Intelligence}, pages 2793--2799, 2017.

\bibitem[Sui et~al.(2015)Sui, Gotovos, Burdick, and Krause]{sui2015safe}
Yanan Sui, Alkis Gotovos, Joel Burdick, and Andreas Krause.
\newblock Safe exploration for optimization with gaussian processes.
\newblock In \emph{International conference on machine learning}, pages
  997--1005. PMLR, 2015.

\bibitem[Sui et~al.(2018)Sui, Zhuang, Burdick, and Yue]{sui2018stagewise}
Yanan Sui, Vincent Zhuang, Joel Burdick, and Yisong Yue.
\newblock Stagewise safe bayesian optimization with gaussian processes.
\newblock In \emph{International conference on machine learning}, pages
  4781--4789. PMLR, 2018.

\bibitem[Usmanova et~al.(2019)Usmanova, Krause, and
  Kamgarpour]{usmanova2019safe}
Ilnura Usmanova, Andreas Krause, and Maryam Kamgarpour.
\newblock Safe convex learning under uncertain constraints.
\newblock In \emph{The 22nd International Conference on Artificial Intelligence
  and Statistics}, pages 2106--2114. PMLR, 2019.

\bibitem[Usmanova et~al.(2020)Usmanova, Krause, and
  Kamgarpour]{usmanova2020safe}
Ilnura Usmanova, Andreas Krause, and Maryam Kamgarpour.
\newblock Safe non-smooth black-box optimization with application to policy
  search.
\newblock In \emph{Learning for Dynamics and Control}, pages 980--989. PMLR,
  2020.

\bibitem[Wang et~al.(2022)Wang, Wagenmaker, and Jamieson]{wang2022best}
Zhenlin Wang, Andrew~J Wagenmaker, and Kevin Jamieson.
\newblock Best arm identification with safety constraints.
\newblock In \emph{International Conference on Artificial Intelligence and
  Statistics}, pages 9114--9146. PMLR, 2022.

\end{thebibliography}

\appendix

\section{Proofs from Section \ref{sec:geom_props}}
\label{apx:geom_proofs}
\subsection{Proof of Proposition \ref{prop:no_empty}}
\begin{proof}
    First we show the ``if" direction.
    By definition, $\mathcal{D}$ has a nonempty interior if and only if there exists a point $x \in \mathcal{D}$ such that $x + \mathcal{B}_2(\epsilon) \subseteq \mathcal{D}$ for some $\epsilon > 0$ where $\mathcal{B}_2(\epsilon) = \{ x \in \mathbb{R}^m : \| x \|_2 < \epsilon \}$ is the open ball of radius $\epsilon$.
    Since $\mathcal{D}$ is closed, this is equivalent to the condition $x + \bar{\mathcal{B}}_2(\epsilon) \subseteq \mathcal{D}$, where $\bar{\mathcal{B}}_2(\epsilon) = \{ x \in \mathbb{R}^m : \| x \|_2 \leq \epsilon \}$.
    Norms on finite-dimensional vector spaces are equivalent, so for some norm $\| \cdot \|$, there exists positive constant $C$ such that for all $y \in \mathcal{B}_{\| \cdot \|}(\tilde{\epsilon})$, $C \| y \|_2 \leq  \| y \|  \leq \tilde{\epsilon}$.
    We can choose $\tilde{\epsilon} = C \epsilon $ such that $\mathcal{B}_{\| \cdot \|}(\tilde{\epsilon}) \subseteq \mathcal{B}_{2}(\epsilon)$, ensuring that $x + \mathcal{B}_{\| \cdot \|}(\tilde{\epsilon}) \subseteq \mathcal{D}$.
    Therefore, we know that $x + v \in \mathcal{D}, \forall v \in \mathcal{B}_{\| \cdot \|}(\Delta)$ for $\Delta = \tilde{\epsilon} > 0$ and hence that $\mathcal{D}_{\Delta}^{\| \cdot \|}$ is nonempty.

    Next, we show the ``only if" direction.
    There exists a $\Delta > 0$ such that $\mathcal{D}_{\Delta}^{\| \cdot \|}$ is nonempty, if and only if there exists a point $x \in \mathcal{D}_{\Delta}^{\| \cdot \|}$.
    The point $x$ must satisfy $x + v \in \mathcal{D}, \forall v \in \bar{\mathcal{B}}_{\| \cdot \|}(\Delta)$.
    It follows that $x + \bar{\mathcal{B}}_{\| \cdot \|}(\Delta) \subseteq \mathcal{D}$.
    From the equivalence of norms discussed previously, it follows that there is some $\epsilon > 0$ such that $\bar{\mathcal{B}}_{2}(\epsilon) \subseteq \bar{\mathcal{B}}_{\| \cdot \|}(\Delta)$ and therefore that $x + \bar{\mathcal{B}}_{2}(\epsilon) \subseteq \mathcal{D}$.
    This implies that $x + \mathcal{B}_{2}(\epsilon) \subseteq \mathcal{D}$ because $\mathcal{B}_{2}(\epsilon) \subset \bar{\mathcal{B}}_{2}(\epsilon)$.
    Therefore, $\mathcal{D}$ has a nonempty interior.
\end{proof}

\subsection{Proof of Proposition \ref{prop:range}}
\begin{proof}  
    Let $E := \{ \Delta \geq 0 : \mathcal{D}_{\Delta}^{\| \cdot \|} \neq \emptyset \}$ and note that $H_{\mathcal{D}}^{\| \cdot \|} = \sup E$ by definition.
    First, we show that $\sup E$ does in fact exist by showing that $E$ is nonempty and is bounded above.
    We know that $E$ is nonempty because $\mathcal{D}$ has a nonempty interior and Proposition \ref{prop:no_empty} tells us that for such sets, there exists a $\Delta > 0$ such that $\mathcal{D}_{\Delta}^{\| \cdot \|}$ is nonempty and therefore $E$ is nonempty.
    Also, note that $E$ is bounded above because $\mathcal{D}$ is bounded (i.e. there exists finite $r$ such that $\| x - y \|_2 \leq r$ for all $x,y \in \mathcal{D}$).
    To see this, first note that for any $\Delta \in E$ (i.e $\Delta$ such that $\mathcal{D}_{\Delta}^{\| \cdot \|} \neq \emptyset$) there exists $x \in \mathcal{D}$ such that $x + v \subseteq \mathcal{D}$ for all $v \in \bar{\mathcal{B}}_{\| \cdot \|}(\Delta)$ and due to equivalence of norms, it holds that $x + \tilde{v} \subseteq \mathcal{D}$ for all $\tilde{v} \in \bar{\mathcal{B}}_{2}(C\Delta)$ for some finite $C > 0$.
    Choosing $\tilde{v} = \mathbf{e}_1 C \Delta$, we have that $x$ and $x + \mathbf{e}_1 C \Delta$ are in $\mathcal{D}$ and hence $\| x - (x + \mathbf{e}_1 C \Delta) \|_2 = C \Delta \leq r$.
    Therefore, $E$ is bounded above.
    Since $E$ is nonempty and bounded above, $\sup E$ exists.

    Next, we show that $H := H_{\mathcal{D}}^{\| \cdot \|} = \sup E$ is in $E$.
    Due to the properties of the supremum, there necessarily exists a $\tilde{\Delta}_\epsilon \in E$ such that $\tilde{\Delta}_\epsilon > H - \epsilon$ for every $\epsilon > 0$.
    Therefore, for every $\epsilon \in (0,H]$ it holds that $\bar{\mathcal{B}}_{\| \cdot \|} (\tilde{\Delta}_\epsilon) \supset \bar{\mathcal{B}}_{\| \cdot \|} ( H - \epsilon)$ and there is some $x$ such that $x + \bar{\mathcal{B}}_{\| \cdot \|}(\tilde{\Delta}_\epsilon) \subseteq \mathcal{D}$.
    It follows that $x + \bar{\mathcal{B}}_{\| \cdot \|}(H - \epsilon) \subset \mathcal{D}$ for every $\epsilon \in (0,H]$.
    We claim that this implies that $x + \mathcal{B}_{\| \cdot \|}(H) \subset \mathcal{D}$.
    Suppose the contrary, i.e. $x + \mathcal{B}_{\| \cdot \|}(H) \not\subset \mathcal{D}$.
    This would imply that there exists $y$ such that $\| y \| < H$ and $x + y \not\in \mathcal{D}$.
    Since $x + y \not\in \mathcal{D}$ and $x + \bar{\mathcal{B}}_{\| \cdot \|}(H - \epsilon) \subset \mathcal{D}$ for every $\epsilon \in (0,H]$, it follows that $y$ cannot be in $\bar{\mathcal{B}}_{\| \cdot \|}(H - \epsilon)$ and therefore $\| y \| \not\leq H - \epsilon$ or equivalently, $\| y \| > H - \epsilon$.
    Therefore, it must be that $H - \epsilon < \| y \| < H$ for every $\epsilon \in (0,H]$.
    Rearranging, it must be that $\epsilon > H - \| y \| > 0$ for every $\epsilon \in (0,H]$, so we can choose $\epsilon = H - \| y \| > 0$ to break the previous statement.
    Therefore, it holds that $x + \mathcal{B}_{\| \cdot \|}(H) \subset \mathcal{D}$.
    Since $\mathcal{D}$ is closed, it is also the superset of $x + \bar{\mathcal{B}}_{\| \cdot \|}(H)$ and therefore, $H$ is in $E$.

    Lastly, we show that all $\Delta \in [0,H]$ are in $E$.
    Since there exists $x \in \mathcal{D}$ such that $x + \bar{\mathcal{B}}_{\| \cdot \|}(H) \subseteq \mathcal{D}$ and $\bar{\mathcal{B}}_{\| \cdot \|}(\Delta) \subseteq \bar{\mathcal{B}}_{\| \cdot \|}(H)$  for all $\Delta \in [0,H]$, it also holds that $x + \bar{\mathcal{B}}_{\| \cdot \|}(\Delta) \subseteq \mathcal{D}$.
\end{proof}

\subsection{Proof of Proposition \ref{prop:sharp_polyt}}
\begin{proof}
    First we show that $\mathcal{D}_{\Delta}^{\| \cdot \|}$ is a polytope within $\mathcal{D}$.
    To do so, we define the halfspace due to the $j$th constraint as $\mathcal{D}_j = \{ x \in \mathbb{R}^m : a_j ^{\top} x \leq b_j \}$ such that $\mathcal{D} = \bigcap_{j \in [p]} \mathcal{D}_j$.
    With some slight abuse of notation, let
        \begin{align*}
        \mathcal{D}_{j,\Delta}^{\| \cdot \|} & := \{ x \in \mathbb{R}^m : a_j ^{\top} (x + v) \leq b_j, \forall v \in  \bar{\mathcal{B}}_{\| \cdot \|} (\Delta) \}\\
        & = \{ x \in \mathbb{R}^m: \max_{v \in \bar{\mathcal{B}}_{\| \cdot \|} (\Delta)} a_j ^{\top} (x + v) \leq b_j \}\\
        & = \{ x \in \mathbb{R}^m : a_j ^{\top} x + \Delta \| a_j \|_{\star} \leq b_j \}\\
        & = \{ x \in \mathbb{R}^m : a_j ^{\top} x \leq b_j - \Delta \| a_j \|_{\star} \},
    \end{align*}
    where $\| \cdot \|_{\star}$ is the dual norm of $\| \cdot \|$.
    Therefore, $\mathcal{D}_{\Delta}^{\| \cdot \|} = \bigcap_{j \in [p]} \mathcal{D}_{j,\Delta}^{\| \cdot \|}$ is a polytope such that each of its constituent constraints are parallel to a corresponding constraint of $\mathcal{D}$, i.e. $\mathcal{D}_{j,\Delta}^{\| \cdot \|}$ is parallel to $\mathcal{D}_j$ for all $j \in [p]$.
    However, it is important to note that some of the constraints $\mathcal{D}_{j,\Delta}^{\| \cdot \|}$ may be redundant.
    Let $E$ be the set of such constraints that are not redundant, i.e. the smallest set $E \subseteq [p]$ such that $\mathcal{D}_{\Delta}^{\| \cdot \|} = \bigcap_{j \in E} \mathcal{D}_{j,\Delta}^{\| \cdot \|}$.
    In order to study the sharpness of $\mathcal{D}$, we will use the polytope $\tilde{\mathcal{D}} := \bigcap_{j \in E} \mathcal{D}_j$.
    The polytope $\tilde{\mathcal{D}}$ is useful because for each non-redundant constraint of $\tilde{\mathcal{D}}$, there is a parallel non-redundant constraint of $\mathcal{D}_\Delta^{\| \cdot \|}$.
    We also know that $\mathcal{D} \subseteq \tilde{\mathcal{D}}$ since
    \begin{equation*}
        \mathcal{D} = \bigcap_{j \in [p]} \mathcal{D}_j = \left( \bigcap_{j \in E} \mathcal{D}_j \right) \cap \left( \bigcap_{j \in [p] \setminus E} \mathcal{D}_j \right) \subseteq \bigcap_{j \in E} \mathcal{D}_j = \tilde{\mathcal{D}}.
    \end{equation*}
    Therefore, it holds that
    \begin{equation*}
        \mathrm{Sharp}_{\mathcal{D}}^{\| \cdot \|} (\Delta) = \max_{x \in \mathcal{D}} \min_{y \in \mathcal{D}_\Delta^{\| \cdot \|}} \| x - y \|_2 \leq \max_{x \in \tilde{\mathcal{D}}} \min_{y \in \mathcal{D}_\Delta^{\| \cdot \|}} \| x - y \|_2.
    \end{equation*}
    Hence, we study a point in $\tilde{\mathcal{D}}$ that has the greatest projection distance, i.e.
    \begin{equation*}
        u \in \argmax_{x \in \tilde{\mathcal{D}}}  \underbrace{\min_{y \in \mathcal{D}_{\Delta}^{\| \cdot \|}} \| y - x  \|_2}_{g(x)}.
    \end{equation*}
    Since $\| y - x  \|_2$ is convex in $(x,y)$, we know that $g(x)$ is convex (\cite{boyd2004convex} Sec. 3.2.5).
    The maximizer of a convex function over a polytope is at a vertex,\footnote{Each point in a polytope can be described as $x = \sum_{i \in [n]} \lambda_i v_i$ where $v_i$ are the vertices and $\lambda_i \geq 0$ with $\sum_{i \in [n]} \lambda_i=1$. By Jensen's inequality, it follows that $f(x) = f(\sum_{i \in [n]} \lambda_i v_i) \leq \sum_{i=[n]} \lambda_i f(v_i) \leq \max_{i \in [n]} f(v_i)$. Since all $v_i$ are in the polytope, the set of maxima must include at least one of them.} and hence $u$ is at some vertex of $\tilde{\mathcal{D}}$.

    Now, we bound the distance between $u$, and the vertex in $\mathcal{D}_\Delta^{\| \cdot \|}$ with the same active constraints as $u$.
    To do so, we first recall the notation from the main text and introduce some new notation.
    The notation $\mathcal{I}_A$ refers to a collection of sets of length $m$ such that for each $\ell = \{i_1, i_2, ..., i_m\} \in \mathcal{I}_A$ the vectors $a_{i_1}, a_{i_2}, ..., a_{i_m}$ are linearly independent.
    We also use the notation $A^\ell = [a_{i_1}\ a_{i_2}\ ...\ a_{i_m}]^\top$ and $b^\ell = [b_{i_1}\ b_{i_2}\ ...\ b_{i_m}]^\top$.
    Note that, for each vertex of a polytope without redundant constraints, there are $m$ linearly independent constraints that are active at that vertex, and they form an element of $\mathcal{I}_A$.
    Accordingly, for a vertex $v$, we use the notation $A^v = [a_{i_1}\ a_{i_2}\ ...\ a_{i_m}]^\top$ and $b^v = [b_{i_1}\ b_{i_2}\ ...\ b_{i_m}]^\top$ where $\ell_v = \{i_1, i_2, ..., i_m\}$ is the set of $m$ linearly independent active constraints at that vertex.
    We now proceed to bound the distance between $u$ and the vertex $v \in \mathcal{D}_\Delta^{\| \cdot \|}$ with the same active constraints as $u$.
    Since the same constraints are tight at both $u$ and $v$, we have that $A^v u = b^v$ and $A^v v = b^v - \Delta \alpha_v$, where $\alpha_v := [\| a_{i_1} \|_{\star} \ ...\ \| a_{i_m}  \|_{\star}]^{\top}$.
    Therefore, we have that $u = (A^v)^{-1} b^v$ and $v = (A^v)^{-1} \left( b^v - \Delta \alpha_v \right)$.
    It follows that
    \begin{align*}
        \| u - v \|_2 & = \| (A^v)^{-1} b^v - (A^v)^{-1} \left( b^v - \Delta \alpha_v \right)\|_2\\
        & = \Delta \| (A^v)^{-1} \alpha_v \|_2\\
        & \leq \Delta \| (A^v)^{-1} \|_2 \| \alpha_v \|_2\\
        & = \Delta \frac{ \sqrt{\sum_{i \in \ell_v} \| a_i \|_{\star}^2}}{\sigma_{min} ( A^v )}
    \end{align*}
    Now, consider the numerator of the above:
    \begin{align*}
        \sqrt{\sum_{i \in \ell_v} \| a_i \|_{\star}^2} & \leq C_{\| \cdot \|} \sqrt{\sum_{i \in \ell_v} \| a_i \|_2^2} = C_{\| \cdot \|} \| A^v \|_F \leq \sqrt{m} C_{\| \cdot \|} \| A^v \|_2 = \sqrt{m} C_{\| \cdot \|} \sigma_{max} (A^v)
    \end{align*}
    where $C_{\| \cdot \|} = \max_{\|x\|_2 = 1} \| x \|_\star$ such that $\| x \|_\star \leq C_{\| \cdot \|} \| x \|_2$ for all $x \in \mathbb{R}^m$.
    Therefore, we have that
    \begin{equation*}
        \| u - v \|_2 \leq \Delta \frac{ \sqrt{\sum_{i \in \ell_v} \| a_i \|_{\star}^2}}{\sigma_{min} ( A^v )} \leq \sqrt{m} C_{\| \cdot \|}  \frac{ \sigma_{max} (A^v)}{\sigma_{min} ( A^v )} \Delta = \sqrt{m} C_{\| \cdot \|} \kappa(A^v) \Delta,
    \end{equation*}
    where $\kappa(A^v)$ is the condition number of $A^v$.
    Since $v$ is in $\mathcal{D}_{\Delta}^{\| \cdot \|}$, we have that 
    \begin{equation*}
        \mathrm{Sharp}_{\mathcal{D}}^{\| \cdot \|} (\Delta) \leq \min_{y \in \mathcal{D}_{\Delta}^{\| \cdot \|}} \| y - u  \|_2 \leq \| v - u  \|_2 \leq \sqrt{m} C_{\| \cdot \|} K  \Delta
    \end{equation*}
    where $K = \max_{\ell \in \mathcal{I}_A} \kappa(A^\ell)$.
    We further characterize $C_{\| \cdot \|}$ as
    \begin{equation*}
        C_{\| \cdot \|} = \max_{\|x\|_2 = 1} \| x \|_\star = \max_{\|x\|_2 = 1} \max_{\| y \| = 1} y^\top x = \max_{\| y \| = 1} \max_{\|x\|_2 = 1} y^\top x = \max_{\| y \| = 1} \| y \|_2.
    \end{equation*}
    Note that $\| x \|_2 \leq C_{\| \cdot \|} \| x \|$ for all $x \in \mathbb{R}^m$.
\end{proof}

\subsection{Proof of Proposition \ref{prop:linear_sets}}
First, we have a lemma that we will need for the proof.
\begin{lemma}
\label{lem:conv_subs}
    For compact sets  $\mathcal{A}, \mathcal{D} \subset \mathbb{R}^m$ that have nonempty interior and satisfy $\mathcal{A} \subseteq \mathcal{D}$, it holds that $\min_{y \in \mathcal{D}_{\Delta}^{\| \cdot \|}} \| y - x \|_2 \leq \min_{y \in \mathcal{A}_{\Delta}^{\| \cdot \|}} \| y - x \|_2$ for all $x \in \mathcal{D}$.
\end{lemma}
\begin{proof}
    First, we show that $\mathcal{A}_{\Delta}^{\| \cdot \|} \subseteq \mathcal{D}_{\Delta}^{\| \cdot \|}$.
    For any $a \in \mathcal{A}_{\Delta}$, it holds that $a + \bar{\mathcal{B}}_{\| \cdot \|} ( \Delta) \subseteq \mathcal{A}$.
    Since $\mathcal{A} \subseteq \mathcal{D}$, it follows that $a + \bar{\mathcal{B}}_{\| \cdot \|} ( \Delta) \subseteq \mathcal{D}$ and hence $a \in \mathcal{D}_{\Delta}^{\| \cdot \|}$.
    Therefore, $\mathcal{A}_{\Delta}^{\| \cdot \|} \subseteq \mathcal{D}_{\Delta}^{\| \cdot \|}$
    and hence $\min_{y \in \mathcal{D}_{\Delta}^{\| \cdot \|}} \| y - x \|_2 \leq \min_{y \in \mathcal{A}_{\Delta}^{\| \cdot \|}} \| y - x \|_2$ for all $x \in \mathcal{D}$.
\end{proof}
Then, we have the proof of Proposition \ref{prop:linear_sets}.

\begin{proof}
In order to bound the sharpness, we choose some arbitrary $\mathcal{F}_x \in F_{\mathcal{D}}(x)$ for each $x \in \mathcal{D}$ and use Lemma \ref{lem:conv_subs} as
\begin{align*}
    \mathrm{Sharp}_{\mathcal{D}}^{\| \cdot \|} (\Delta) & = \max_{x \in \mathcal{D}} \min_{y \in \mathcal{D}_{\Delta}^{\| \cdot \|}}  \left \| y - x \right \|_2\\
    & \leq \max_{x \in \mathcal{D}} \min_{y \in \mathcal{F}_{x,\Delta}^{\| \cdot \|}}  \left \| y - x \right \|_2\\
    & \leq \max_{x \in \mathcal{D}} \max_{z \in \mathcal{F}_x} \min_{y \in \mathcal{F}_{x,\Delta}^{\| \cdot \|}}  \left \| y - z \right \|_2\\
    & = \max_{x \in \mathcal{D}} \mathrm{Sharp}_{\mathcal{F}_x}^{\| \cdot \|} (\Delta)\\
    & \leq \Delta \sqrt{m} C_{\| \cdot \|}   \max_{x \in \mathcal{D}} K_{\mathcal{F}_x}
\end{align*}
which is valid for $\Delta \in [0, \bar{H}_{\mathcal{F}}^{\| \cdot \|}]$, where $\bar{H}_{\mathcal{F}}^{\| \cdot \|} := \min_{x \in \mathcal{D}} H_{\mathcal{F}_x}^{\| \cdot \|} > 0$ given that every $\mathcal{F}_x$ has nonempty interior by definition and Proposition \ref{prop:no_empty}.
We use the notation $\bar{K}_{\mathcal{F}} = \max_{x \in \mathcal{D}} K_{\mathcal{F}_x}$.

We can then use the boundedness of $\mathcal{D}$ to establish a linear bound for the remainder of the domain of the sharpness, i.e. for $\Delta \in (\bar{H}_{\mathcal{F}}^{\| \cdot \|},H_{\mathcal{D}}^{\| \cdot \|}]$.
From the boundedness of $\mathcal{D}$, we have that the diameter $r_{\mathcal{D}} = \sup_{x,y \in \mathcal{D}} \| x - y \|_2$ is finite and hence we have the trivial sharpness bound $\mathrm{Sharp}_{\mathcal{D}}^{\| \cdot \|} (\Delta) \leq r_{\mathcal{D}}$.
This gives the linear bound $\mathrm{Sharp}_{\mathcal{D}}^{\| \cdot \|} (\Delta) \leq \frac{r_{\mathcal{D}}}{\bar{H}_{\mathcal{F}}^{\| \cdot \|}} \Delta$ for all $\Delta \in (\bar{H}_{\mathcal{F}}^{\| \cdot \|},H_{\mathcal{D}}^{\| \cdot \|}]$.
Therefore, we have that 
\begin{equation*}
    \mathrm{Sharp}_{\mathcal{D}}^{\| \cdot \|} (\Delta) \leq \max \left \{ \sqrt{m} C_{\| \cdot \|} \bar{K}_{\mathcal{F}}, \frac{r_{\mathcal{D}}}{\bar{H}_{\mathcal{F}}^{\| \cdot \|}} \right \} \Delta \leq \sqrt{m} \bar{C}_{\| \cdot \|} \max \left \{  \bar{K}_{\mathcal{F}}, \frac{r_{\mathcal{D}}}{\bar{H}_{\mathcal{F}}^{\| \cdot \|}} \right \} \Delta,
\end{equation*}
where $\bar{C}_{\| \cdot \|} = \max( C_{\| \cdot \|},1)$.
\end{proof}

\subsection{Proof of Corollary \ref{cor:sharp_conv}}
\begin{proof}
    Given Proposition \ref{prop:linear_sets}, it is sufficient to show that a polytope with nonempty interior can be constructed to contain each point in $\mathcal{D}$ while being a subset of $\mathcal{D}$.
    Since $\mathcal{D}$ has a nonempty interior, we can construct a polytope $\mathcal{A}$ with nonempty interior in $\mathcal{D}$ (e.g. a hypercube).
    Then, from the convexity of $\mathcal{D}$, we can see that the polytope $\bar{\mathcal{A}} = \mathrm{conv}(\mathcal{A},x)$ contains $x$ and lies in $\mathcal{D}$ for each $x \in \mathcal{D}$.
    Therefore, a convex set $\mathcal{D}$ is polytope-sharp and the result in Proposition \ref{prop:linear_sets} applies. 
\end{proof}

\subsection{Additional Properties}
\begin{proposition}
\label{prop:subset}
    For compact sets $\mathcal{A},\mathcal{D} \subset \mathbb{R}^m$ with nonempty interiors, where $\mathcal{A} \subseteq \mathcal{D}$, we have that $H_{\mathcal{A}}^{\| \cdot \|} \leq H_{\mathcal{D}}^{\| \cdot \|}$.
\end{proposition}
\begin{proof}
    Consider any $\Delta \in [0,H_{\mathcal{A}}^{\| \cdot \|}]$.
    From Prop. \ref{prop:range} and the definition of a shrunk set, there exists $x \in \mathcal{A}^{\| \cdot \|}_{\Delta}$ such that $x + \mathcal{B}_{\| \cdot \|}(\Delta) \subseteq \mathcal{A}$.
    Since $\mathcal{A} \subseteq \mathcal{D}$, we have that $x + \mathcal{B}_{\| \cdot \|}(\Delta) \subseteq \mathcal{D}$ and hence $\mathcal{D}^{\| \cdot \|}_{\Delta}$ is nonempty and $\Delta \leq H_{\mathcal{D}}^{\| \cdot \|}$.
    Since this applies for all $\Delta \in [0,H_{\mathcal{A}}^{\| \cdot \|}]$, it follows that $H_{\mathcal{A}}^{\| \cdot \|} \leq H_{\mathcal{D}}^{\| \cdot \|}$.
\end{proof}

\section{Proof of Theorem \ref{thm:main}}
\label{apx:thm_main}
We use the decomposition of the instantaneous regret specified in \eqref{eqn:reg_decomp}.
Accordingly, we discuss Term I in Section \ref{sec:term_i}, Term II in Section \ref{sec:term_ii} and the complete regret bound in Section \ref{sec:comp_bound}.

\subsection{Term I (Proof of Lemma \ref{lem:term_i})}
\label{sec:term_i}

Before getting to the proof of Lemma \ref{lem:term_i}, we need a lemma from \cite{amani2019linear}.

\begin{lemma}
\label{lem:min_eig}
    (Lemma 1 in \cite{amani2019linear}) We have with probability at least $1 - \delta$ that
    \begin{equation}
        \lambda_{\text{min}} (V_{T'}) \geq \nu + \frac{\lambda_- T'}{2},
    \end{equation}
    for $T' \geq t_\delta := \frac{8 L^2}{\lambda_-} \log(\frac{d}{\delta})$.
\end{lemma}

We also need to show that the sharpness of $\mathcal{Y}$ is well defined, for which we need the following lemma.

\begin{lemma}
\label{lem:nonemp_int}
    For a set $\mathcal{D} \subset \mathbb{R}^d$ with nonempty interior and full rank matrix $N \in \mathbb{R}^{n \times d}$ with $n \leq d$, the set $\mathcal{H} = N \mathcal{D}$ has a nonempty interior.
\end{lemma}
\begin{proof}
    Consider a point $x_0$ in the interior of $\mathcal{D}$.
    By definition, there exists an open ball of radius $\epsilon > 0$ that is centered at $x_0$ and lies within $\mathcal{D}$.
    We can choose $n+1$ points within that ball, $x_0,x_1,...,x_n$, which defines the polytope $\mathrm{conv}(x_0,x_1,...,x_n) \subset \mathcal{D}$.
    The image of these points under $N$ is defined as $v_i = N x_i$ for all $i \in \{0,1,...,n\}$, such that $\mathrm{conv}(v_0,v_1,...,v_n) \subset \mathcal{H}$.
    If the vectors $v_0 - v_1, v_0 - v_2, ..., v_0 - v_n$ can be chosen to be linearly independent, then $\mathrm{conv}(v_0,v_1,...,v_n)$ has nonempty interior and hence $\mathcal{H}$ has nonempty interior.
    To see that this is possible, first note for all $i \in \{ 1,2,...,n \}$ we can choose $x_i$ such that $x_0 - x_i = \frac{\epsilon}{2 } w_i$ for any unit vector $w_i$ of dimension $d$.
    It follows that $z_i = v_0 - v_i = N (x_0 - x_i) = \frac{\epsilon}{2 } N w_i$ for all $i \in \{ 1,2,...,n \}$.
    Since the rank of $N$ is $n$, the dimension of the image space of $N$ is $n$ and hence there exists a $w_1, w_2, ..., w_n$ such that $z_1, z_2, ..., z_n$ are linearly independent (i.e. a basis for the image space).
\end{proof}
From Assumption \ref{ass:init_set}, we also have that $\mathcal{G}^0$ has nonempty interior which implies that $\mathcal{X}$ has nonempty interior as $\mathcal{X} \supset \mathcal{G}^0$.
Since we can write $\mathcal{Y} = \Theta_* \mathcal{X}$ and $\mathcal{X}$ has nonempty interior, then by Lemma \ref{lem:nonemp_int}, we know that $\mathcal{Y}$ has nonempty interior.
We also know that $\mathcal{Y}$ is compact because $\mathcal{A}$ and $\mathcal{E}$ are compact and a finite dimensional linear operator preserves compactness.
Since $\mathcal{Y}$ is compact and has nonempty interior, the sharpness of $\mathcal{Y}$ is well defined.

We can then give the proof of Lemma \ref{lem:term_i}.

\begin{proof}
Without further reference to it, we take $\Theta_*$ to be in $\mathcal{C}_t$ for all $t \in [T]$, which holds with probability at least $1 - \delta$ by Theorem \ref{thm:conf_set}.
Similarly, we take $\lambda_{\text{min}} (V_{T'}) \geq \nu + \frac{\lambda_- T'}{2}$, which holds with probability at least $1 - \delta$ given Lemma \ref{lem:min_eig} and that $T' \geq t_\delta$.
By the union bound, both $\Theta_* \in \mathcal{C}_t, \forall t \in [T]$ and $\lambda_{\text{min}} (V_{T'}) \geq \nu + \frac{\lambda_- T'}{2}$ jointly hold with probability at least $1 - 2 \delta$.

First, we define an expanded version of the confidence set for the unknown parameter,
\begin{equation*}
    \tilde{\mathcal{C}}_t := \left\{ [\theta^1\ \theta^2\ ...\ \theta^n]^\top \in \mathbb{R}^{n \times d} : \left\| \theta^i - \theta_*^i \right\|_{V_t} \leq 2 \sqrt{\beta_t}, \forall i \in [n] \right\} \supseteq \mathcal{C}_t.
\end{equation*}
This holds because for all $\Theta$ in $\mathcal{C}_t$, we have for all $i \in [n]$ that 
\begin{equation*}
    \left\| \theta^i - \theta_*^i \right\|_{V_t} = \left\| \theta^i - \hat{\theta}_t^i + \hat{\theta}_t^i - \theta_*^i \right\|_{V_t} \leq \left\| \theta^i - \hat{\theta}_t^i \right \|_{V_t} + \left\| \hat{\theta}_t^i - \theta_*^i \right\|_{V_t} \leq 2 \sqrt{\beta_t}.
\end{equation*}
We use this to the define a shrunk safe action set:
\begin{equation}
    \tilde{\mathcal{G}}_t := \{ x \in \mathcal{A} : \Theta x \in \mathcal{E}, \forall \Theta \in \tilde{\mathcal{C}}_t \} \subseteq \mathcal{G}_t
\end{equation}
We then use Lemma \ref{lem:min_eig} to define a further shrunk safe price set, such that for $t \geq T'$, we have that
\begin{equation}
\label{eqn:shrunk_price}
    \bar{\mathcal{G}} = \{ x \in \mathcal{A} : \Theta_* x + v \in \mathcal{E}, \forall v \in \bar{\mathcal{B}}_{\infty}(\ell) \} \subseteq \tilde{\mathcal{G}}_t,
\end{equation}
where $\ell := \frac{2  \sqrt{2 \beta_T} L }{\sqrt{2 \nu + \lambda_- T'}}$.
To see that \eqref{eqn:shrunk_price} holds, note that for any $x$ and $\Theta \in \tilde{\mathcal{C}}_t$ we have for all $i \in [n]$ that
\begin{align*}
    x^\top \theta^i & \in \left[x^\top \theta^i_* - 2 \sqrt{\beta_t} \| x \|_{V_t^{-1}}, x^\top \theta^i_* + 2 \sqrt{\beta_t} \| x \|_{V_t^{-1}} \right]\\
    & \subseteq \left[ x^\top \theta^i_* - \frac{2 \sqrt{\beta_T} L}{\sqrt{\lambda_{min}(V_t)}}, x^\top \theta^i_* + \frac{2 \sqrt{\beta_T} L}{\sqrt{\lambda_{min}(V_t)}} \right]\\
    & \subseteq \left[ x^\top \theta^i_* - \frac{2  \sqrt{2 \beta_T} L }{\sqrt{2 \nu + \lambda_- T'}}, x^\top \theta^i_* + \frac{2  \sqrt{2 \beta_T} L }{\sqrt{2 \nu + \lambda_- T'}} \right]\\
    & = \left[ x^\top \theta^i_* - \ell, x^\top \theta^i_* + \ell \right].
\end{align*}
Therefore any $x$ such that $\Theta_* x + v$ is in $\mathcal{E}$ for all $v \in \bar{\mathcal{B}}_{\infty}(\ell)$ will also ensures that $\Theta x$ is in $\mathcal{E}$ for all $\Theta \in \tilde{\mathcal{C}}_t$ and hence $\bar{\mathcal{G}} \subseteq \tilde{\mathcal{G}}_t$.
Then we have,
\begin{align*}
    \bar{\mathcal{Y}} := \{ \Theta_* x : x \in \bar{\mathcal{G}} \} & = \{ \Theta_* x : x \in \mathcal{A} \} \cap \{ y : y + v \in \mathcal{E}, \forall v \in \bar{\mathcal{B}}_{\infty}(\ell) \}\\
    & = \{ \Theta_* x : x \in \mathcal{A} \} \cap \mathcal{E}_{\ell}^\infty \\
    & \supseteq \mathcal{Y}_{\ell}^\infty,
\end{align*}
where $\mathcal{Y} := \{ \Theta_* x : x \in \mathcal{A} \} \cap \mathcal{E}$.
We will also need the definition $\tilde{\mathcal{Y}}^t := \{ \Theta x : x \in \mathcal{G}_t, \Theta \in \mathcal{C}_t \}$.
Note that $\tilde{\mathcal{Y}}^t \supseteq  \bar{\mathcal{Y}} \supseteq \mathcal{Y}_{\ell}^\infty$ for $t \geq T'$.

In order to project on to $\mathcal{Y}_{\ell}^\infty$, we need that it is nonempty.
It is nonempty if the safe exploration phase is long enough such that $\ell \leq H_{\mathcal{Y}}^\infty$.
We can therefore ensure that $T'$ is sufficiently large as follows.
\begin{align*}
    \ell & = \frac{2  \sqrt{2 \beta_T} L }{\sqrt{2 \nu + \lambda_- T'}} \leq H_{\mathcal{Y}}^\infty\\
    T' & \geq \frac{8 \beta_T L^2}{\lambda_- \left( H_{\mathcal{Y}}^\infty \right)^2} - \frac{2 \nu}{\lambda_-} =: t_h
\end{align*}
Together with Lemma \ref{lem:min_eig}, we need that $T' \geq \max ( t_\delta, t_h )$.

We can now prove the statement of the lemma directly.
Since \eqref{eqn:optim_update} is optimistic over all $y$ in $\tilde{\mathcal{Y}}^t$ and $\tilde{\mathcal{Y}}^t \supseteq \mathcal{Y}_{\ell}^\infty$, we have that $f(\tilde{y}_t) \geq f(\bar{y})$, where $\tilde{y}_t := \tilde{\Theta}_t x_t$, $y_* = \Theta_* x_*$ and $\bar{y} \in \argmin_{y \in \mathcal{Y}_{\ell}^\infty} \| y_* - y \|_2$.
This yields
\begin{align*}
    \text{Term I} & := f(y_*) - f(\tilde{y}_t)\\
    & \leq f(y_*) - f(\bar{y})\\
    & \leq | f(y_*) - f(\bar{y}) |\\
    & \leq M \| y_* - \bar{y} \|\\
    & \leq M \mathrm{Sharp}_{\mathcal{Y}}^{\infty} (\ell)\\
    & = M \mathrm{Sharp}_{\mathcal{Y}}^{\infty} \left( \frac{2  \sqrt{2 \beta_T} L }{\sqrt{2 \nu + \lambda_- T'}} \right)
\end{align*}
The statement of the lemma immediately follows.
\end{proof}

\subsection{Term II}
\label{sec:term_ii}
First, we need a lemma from \cite{abbasi2011improved}.
\begin{lemma}
\label{lem:elip_pot}
    (Lemma 11 in \cite{abbasi2011improved}) For  $\{ x_t \}_{t=1}^\infty$ with $\nu > 0$ and $V_t = \nu I + \sum_{s=1}^t x_s \left[x_s\right]^\top$, we have that
    \begin{equation*}
        \sum_{t=1}^T \min( \| x_t \|_{[V_{t-1}]^{-1}}^2 , 1  ) \leq 2(d \log((\textrm{trace}(\nu I) + T L^2)/d) - \log \text{det} (\nu I))
    \end{equation*}
    when $\| x_t \|_2 \leq L$ for all $t \in [T]$.
\end{lemma}
We then bound Term II with the following lemma.
\begin{lemma}
\label{lem:term_ii}
    Let Assumptions \labelcref{ass:bounds,ass:init_set,ass:lipsch,ass:noise} hold. For $T > T' \geq \max(t_h, t_\delta)$, $\mathrm{Term\ II}$ is bounded as
    \begin{equation*}
        \sum_{t=T'+1}^T \mathrm{Term\ II} \leq M \max (LS,1) \sqrt{n 8 \beta_T (T - T') d \log( 1 + T L^2/ (d \nu) )}.
    \end{equation*}
    with probability at least $1 - \delta$.
\end{lemma}
\begin{proof}
As in the proof of Lemma \ref{lem:term_i}, we take $\Theta_* \in \mathcal{C}_t, \forall t \in [T]$ and $\lambda_{\text{min}} (V_{T'}) \geq \nu + \frac{\lambda_- T'}{2}$, which jointly holds with probability at least $1 - 2 \delta$.

Using Assumption \ref{ass:lipsch}, we have for $t > T'$ that
\begin{align*}
    \mathrm{Term\ II} := & f(\tilde{\Theta}_t x_t) - f(\Theta_* x_t)\\
    \leq &  | f(\tilde{\Theta}_t x_t) - f(\Theta_* x_t) |\\
    \leq &  M \| \tilde{\Theta}_t x_t - \Theta_* x_t \|_2 \\
    \leq & M \sqrt{n} \max_{i \in [n]} \left( x_t^{\top} \tilde{\theta}^i_t - x_t^{\top} \theta^i_{*} \right)
\end{align*}
Let $\bar{r}_t^{II} :=  x_t^{\top} \tilde{\theta}^i_t - x_t^{\top} \theta^i_{*}$ such that
\begin{align*}
    \bar{r}_t^{II} := & x_t^{\top} \tilde{\theta}^i_t - x_t^{\top} \theta^i_{*}\\
    \leq & \left \| x_t \right \|_{[V_{t-1}]^{-1}} \left \| \tilde{\theta}^i_t  - \theta^i_{*} \right \|_{V_{t-1}}\\
    \leq & \left \| x_t \right \|_{[V_{t-1}]^{-1}} \left \| \tilde{\theta}^i_t -  \hat{\theta}_t^i +   \hat{\theta}_t^i - \theta^i_{*} \right \|_{V_{t-1}}\\
    \leq & 2 \sqrt{\beta_t} \left \| x_t \right \|_{[V_{t-1}]^{-1}}.
\end{align*}
Due to the pure exploration phase, we have that
\begin{equation*}
   2 \sqrt{\beta_t} \left \| x_t \right \|_{[V_{t-1}]^{-1}} \leq \frac{2 \sqrt{\beta_T} L}{\sqrt{\lambda_{min}(V_{t-1})}} \leq \frac{2  \sqrt{2 \beta_T} L }{\sqrt{2 \nu + \lambda_- T'}} \leq H_{\mathcal{Y}}^\infty.
\end{equation*}
Therefore, the following bound applies when $T$ is large enough such that $\beta_T \geq 1$.
\begin{align*}
    \bar{r}_t^{II} \leq & \min \left ( H_{\mathcal{Y}}^\infty, 2 \sqrt{\beta_t} \left \| x_t \right \|_{[V_{t-1}]^{-1}} \right )\\
    \leq & 2 \sqrt{\beta_T} \max (H_{\mathcal{Y}}^\infty,1) \min \left ( 1, \left \| x_t \right \|_{[V_{t-1}]^{-1}} \right )
\end{align*}
We can then use Lemma \ref{lem:elip_pot} as
\begin{align*}
    \sum_{t=T'+1}^T [\bar{r}_t^{II}]^2 = & 4 \beta_T \max ((H_{\mathcal{Y}}^\infty)^2,1) \sum_{t=T'+1}^T \min \left ( 1, \left \| x_t \right \|_{[V_{t-1}]^{-1}}^2 \right )\\
    \leq & 4 \beta_T \max ((H_{\mathcal{Y}}^\infty)^2,1) \sum_{t=1}^T \min \left ( 1, \left \| x_t \right \|_{[V_{t-1}]^{-1}}^2 \right )\\
     \leq & 8 \beta_T \max ((H_{\mathcal{Y}}^\infty)^2,1) (d \log((\textrm{trace}(\nu I) + T L^2)/d) - \log \text{det} (\nu I))\\
    = & 8 \beta_T \max ((H_{\mathcal{Y}}^\infty)^2,1) (d \log((d \nu + T L^2)/d) -  d \log(\nu))\\
    = & 8 \beta_T \max ((H_{\mathcal{Y}}^\infty)^2,1) d \log( 1 + T L^2/ (d \nu) ).
\end{align*}
Therefore, applying Cauchy-Schwarz yields
\begin{align*}
    \sum_{t=T'+1}^T \bar{r}_t^{II} \leq & \sqrt{ (T - T') \sum_{t=T' + 1}^T [r_{t,i}^{II}]^2 }\\
    \leq & \max (H_{\mathcal{Y}}^\infty,1) \sqrt{ 8 \beta_T (T - T') d \log( 1 + T L^2/ (d \nu) )}.
\end{align*}

We then have that
\begin{equation*}
    \sum_{t=T'+1}^T \mathrm{Term\ II} \leq M \max (H_{\mathcal{Y}}^\infty,1) \sqrt{n d 8 \beta_T (T - T') \log( 1 + T L^2/ (d \nu) )}.
\end{equation*}
\end{proof}

\subsection{Complete Regret Bound}
\label{sec:comp_bound}

Finally, we can prove Theorem \ref{thm:main}.
\begin{proof}
    First, we have the trivial bound on the instantaneous regret as
    \begin{equation*}
        r_t := f(\Theta_* x_*) - f(\Theta_* x_t) \leq M \| \Theta_* x_* - \Theta_* x_t\|_2 \leq M \sqrt{n} \max_{i \in [n]} \left( x_*^\top \theta_*^i - x_t^\top \theta_*^i \right) \leq 2 M \sqrt{n} LS
    \end{equation*}
    We use this trivial bound for the first $T'$ time steps.
    Therefore, we can decompose the regret with respect to time and use Lemma \ref{lem:term_i} and Lemma \ref{lem:term_ii} to get
    \begin{align*}
        R_T & = \sum_{t=1}^{T'} r_t + \sum_{t=T'+1}^T (\mathrm{Term\ I}) + \sum_{t=T'+1}^T (\mathrm{Term\ II}) \\
        & = 2 M \sqrt{n} LS T' +  M \max (H_{\mathcal{Y}}^\infty,1) \sqrt{n 8 \beta_T (T - T') d \log( 1 + T L^2/ (d \nu) )} \\
        & \quad  + M (T - T') \mathrm{Sharp}_{\mathcal{Y}}^{\infty} \left( \frac{2  \sqrt{2 \beta_T} L }{\sqrt{2 \nu + \lambda_- T'}} \right)
    \end{align*}
\end{proof}

\section{Numerical Experiments}
\label{apx:exps}
The numerical experiments were performed in a setting of action dimension $d = 3$ and response dimension $n=3$.
The reward function is linear, i.e. $f(y) = a^{\top} y$ for some $a \in \mathbb{R}^n$.
To methodically study the impact of sharpness, we use safety sets of the form $\mathcal{E}_b = \{ y \in \mathbb{R}^n : \| \mathrm{diag}(b) y \|_1 \leq 1 \}$ for some $b \in \mathbb{R}^n$.
By appropriately choosing $b$, we can modify the sharpness of the set without changing the maximum shrinkage.
The action set $\mathcal{A}$ is chosen to be non-restrictive such that $\mathcal{Y} = \mathcal{X}$.

The three safety sets that we study are $\mathcal{E}_{b_1}$, $\mathcal{E}_{b_2}$ and $\mathcal{E}_{b_3}$, where $b_1 = [0.1,0.1,0.1]$, $b_2 = [0.1/2,0.1,0.1]$ and $b_3 = [0.1/3,0.1,0.1]$.
Note that $\mathcal{E}_{b_1}$, $\mathcal{E}_{b_2}$ and $\mathcal{E}_{b_3}$ all have a $2$-norm maximum shrinkage of $10$.
To isolate the impact of sharpness, we ensure that the optimal response $y_* = \theta_* x_*$ is in the ``sharpest" corner of the polytope by choosing $a = \mathbf{e}_1$.
The parameter $\Theta$ is sampled randomly by choosing each element $\Theta^{i,j} \sim U[-1,1]$ for all $i \in [n]$ and $j \in [d]$.
Each element of the noise $\epsilon_t$ is sampled $\epsilon_t^i \sim N(\sigma^2)$ for all $i \in [n]$, where $\sigma=10^{-3}$.
Also, we choose $T = 10^3$.
Additionally, $S = \| \Theta_* \|_2 + 0.1$ and $L = \max_{b \in \{b_1,b_2,b_3\}} \max_{x \in \mathcal{E}_b} \| x \|_2 + 0.1$.
The algorithm parameters are $\nu = 0.1$, $T' = T^{2/3}=10^2$, and $\delta = 0.01$.

At each time step, the optimistic action in \eqref{eqn:optim_update} is calculated using the $\ell_1$ confidence set that is an outer approximation of $\mathcal{C}_t$ as used in \cite{dani2008stochastic} and \cite{amani2019linear}.
Specifically, we use the method detailed in \cite{amani2019linear} which we summarize as follows.
Since $f$ is linear, \eqref{eqn:optim_update} can be solved by enumerating the vertices of $\ell_1$ confidence set and optimizing the reward given the parameter at each vertex.
The optimal reward is then the maximum of the rewards at each vertex and the optimistic action is the maximizing action for this reward.

\end{document}